\newcommand{\submission}{}
\documentclass[11pt]{article}
\PassOptionsToPackage{compress}{natbib}

\usepackage[preprint]{nips_2018}

\usepackage{amsmath}
\usepackage{amsthm}
\usepackage{amssymb}

\usepackage{times,wrapfig,amsfonts, bm,color,enumitem,algorithm,algpseudocode}
\usepackage{microtype}
\usepackage{booktabs} 
\usepackage{hyperref}
\usepackage{amsfonts}
\usepackage{mathtools}
\usepackage{bbm}
\usepackage{comment}
\usepackage{nicefrac}
\usepackage{xcolor} 
\usepackage{esint}
\usepackage{tikz}
\usetikzlibrary{spy,calc}
\usepackage{xparse}

\usepackage[font=small,justification=justified,
   format=plain]{caption}

\newtheorem{thm}{Theorem}[section]

\newtheorem{claim}{Claim}[section]
\newtheorem{cor}[thm]{Corollary}

\newtheorem{clemma}{CLemma}  
\newtheorem{cproof}{CProof}

\excludecomment{clemma}
\excludecomment{cproof}
\excludecomment{cclaim}

\newcommand{\st}{\text{ s.t.} }
\newcommand{\Th}{\text{th} }
\newcommand{\ie}{\textit{i.e.,} }

\newcommand{\sign}{\operatorname{sign}}

\newcommand{\norm}[1]{\left\lVert{#1}\right\rVert}
\newcommand{\abs}[1]{\left\lvert{#1}\right\rvert}
\newcommand{\ip}[2]{\left\langle {#1}, {#2} \right\rangle}

\DeclareMathOperator*{\argmin}{arg\,min}
\DeclareMathOperator*{\conv}{conv}

\newcommand{\iter}[1]{\left[{#1}\right]} 
\newcommand{\R}{\mathbb{R}}

\newcommand{\N}{\mathbb{N}}

\definecolor{cadmiumgreen}{rgb}{0.0, 0.5, 0.24}

\newcommand\notinsubmission[1]{\ifdefined\submission {} \else {{#1}} \fi}

\newcommand{\dnote}[1]{\notinsubmission{{\color{blue}[Daniel: #1]}}}
\newcommand{\inote}[1]{\notinsubmission{{\color{orange}[Itay: #1]}}}
\newcommand{\nnote}[1]{\notinsubmission{{\color{cadmiumgreen}[Nati: #1]}}}
\newcommand{\pnote}[1]{\notinsubmission{{\color{magenta}[Pedro: #1]}}}

\newcommand{\temp}[1]{\notinsubmission{{\color{purple}#1}}}
\newcommand{\todo}[1]{\notinsubmission{{\color{red}todo: #1}}}
\newcommand{\removed}[1]{}
\newcommand{\natinote}[1]{\notinsubmission{\nnote{#1}}}

\renewcommand{\eqref}[1]{(\ref{#1})}
\newcommand{\secref}[1]{Section~\ref{#1}}
\newcommand{\figref}[1]{Figure~\ref{#1}}

\newcommand{\thmref}[1]{Theorem~\ref{#1}}

\newcommand{\appref}[1]{Appendix~\ref{#1}}

\newcommand{\corref}[1]{Corollary~\ref{#1}}

\newcommand{\lemref}[1]{Lemma~\ref{#1}}

\newcommand{\func}[1]{h_{#1}}
\newcommand{\loss}{\ell}
\newcommand{\ThetaClass}{\Theta}
\newcommand{\AlphaClass}{\mathcal{A}}
\newcommand{\optf}{{f^{*}}}
\newcommand{\pwlf}{\tilde{f}}
\newcommand{\intbias}{c}
\def\relu#1{\left[{#1}\right]_+}

\newcommand{\Sd}{\mathbb{S}^{d-1}}  
\newcommand{\subF}{{\mathcal{F}}}
\newcommand{\oR}{\overline{R}}
\newcommand{\RP}{P}

\def\samples{N}
\def\smplIdx{n}
\newcommand{\Lp}[1]{\ell_{#1}} 
\newcommand{\featureSpace}{\mathbb{R}^d} 
\newcommand{\Eucost}{C}
\newcommand{\lyrIdx}{l}
\newcommand{\params}{\theta}
\newcommand{\msrDomain}{\temp{\Omega}}

\newcommand{\coef}[1][]{\wvec_{#1}^{\left(L\right)}}
\newcommand{\optcoef}[1][]{{\wvec^*_{#1}}^{\left(L\right)}}
\newcommand{\subnetNum}{k}
\newcommand{\matSize}{m}
\def\order{\nicefrac{2}{L}}

\def\prn#1{\left({#1}\right)} 
\newcommand{\half}{\frac{1}{2}}
\mathchardef\hyphen="2D

\newcommand*\diff{\mathop{}\!\mathrm{d}}

\newcommand{\myvec}[1]{{\mathbf{#1}}}
\newcommand{\xvec}{\myvec{x}}
\newcommand{\wvec}{\myvec{w}}
\newcommand{\bvec}{\myvec{b}}
\newcommand{\wmat}{W}
\newcommand{\tv}{\mathcal V^{(2)}}
\newtheorem{lemma}{Lemma}[section]
\def\subfunc#1#2{v\left({#1}, #2\right)}
\def\weights{\mathcal{W}}
\def\lweights#1#2{\wmat^{\left(#1\right)}_{#2}}
\def\ulweights#1#2{\bar{\wmat}^{\left(#1\right)}_{#2}}
\def\optlweights#1#2{\wmat^{*,\left(#1\right)}_{#2}}
\def\optulweights#1#2{\bar{\wmat}^{*,\left(#1\right)}_{#2}}
\def\unit{\bar{\weights}}
\def\sphere{\mathcal{S}}

\newcommand{\valpha}[1][]{\ifthenelse{\equal{#1}{}}{\boldsymbol{\alpha}}{\boldsymbol{\alpha}{(#1)}}}
\newcommand{\vbeta}[1][]{\ifthenelse{\equal{#1}{}}{\boldsymbol{\beta}}{\boldsymbol{\beta}{(#1)}}}
\newcommand{\vb}[1][]{\ifthenelse{\equal{#1}{}}{\boldsymbol{b}}{\boldsymbol{b}{(#1)}}}

\def\lifcost#1#2#3{\mathcal{\mathcal{B}}_{#1}^{#2}\prn{#3}} 

\def\ouralpha#1#2{\valpha^{#1}_{#2}}
\def\lspace{\temp{\mathcal{L}^\infty\prn{?}}}

\title{How do infinite width bounded norm networks look in function space?}

\usepackage{times}

\author{
\textbf{Pedro Savarese$^1$} \hfill \texttt{savarese@ttic.edu} \\
\AND
\textbf{Itay Evron$^2$} \hfill \texttt{evron.itay@gmail.com} \\
\AND
\textbf{Daniel Soudry$^2$}  \hfill \texttt{daniel.soudry@gmail.com} \\
\AND
\textbf{Nathan Srebro$^1$} \hfill \texttt{nati@ttic.edu} \\
\AND
\vskip 0.15in
$^1$ Toyota Technological Institute at Chicago, Chicago IL, USA \hfill $ $ \\
$^2$ Department  of  Electrical  Engineering,  Technion,  Haifa, Israel \hfill $ $ \\
}

\begin{document}

\maketitle

\begin{abstract}%
   We consider the question of what functions can be captured by ReLU networks with an unbounded number of units (infinite width), but where the overall network Euclidean norm (sum of squares of all weights in the system, except for an unregularized bias term for each unit) is bounded; 
   or equivalently what is the minimal norm required to approximate a given function.  
   For functions $f:\R\rightarrow\R$ and a single hidden layer, we show that the minimal network norm for representing $f$ is 
   $\max(\,\int\! \abs{f''(x)}\diff x\, ,\,\abs{f'(-\infty) + f'(+\infty)}\,)$, 
   and hence the minimal norm fit for a sample is given by a linear spline interpolation.  
   \notinsubmission{We also investigate deeper bounded-norm infinite-width architectures and show that they correspond to minimizing a non-convex bridge penalty.}
\end{abstract}

\removed{
\begin{keywords}%
  List of keywords%
\end{keywords}
}

\section{Introduction}

\label{sec-intro}

Empirical and theoretical results suggest that neural network models
used in practice are not constrained by their size (number of units,
or number of weights) but perhaps achieve complexity control by controlling the magnitude of the weights
\cite[e.g.][]{bartlett1997valid,neyshabur2014search,zhang2016understanding}, either explicitly or implicitly \citep{soudry2018journal,gunasekar2018implicit}.

  In fact, it is reasonable to think of networks as essentially having
  infinite size (having an infinite number of units), and controlled
  only through some norm of the weights.  Using infinite (or
  unbounded) size networks we can approximate virtually any
  function, and so in training such a model we are essentially
  searching over the space of all functions.  
  
  Learning can be thought
  of as searching over the entire function space for a function with
  small \textit{representation cost}, given by the minimal norm required to represent it in the chosen
  infinite size architecture.  Understanding learning with infinite
  (or unbounded) size networks thus relies crucially on understanding
  this ``representation cost'', which is the actual inductive bias of
  learning.

  Equivalently, we can think of this question as asking what functions
  can be represented, or approximated arbitrarily well, with an
  infinite-size bounded-norm network.  There has been considerable work for
  over three decades on the question of what functions can be
  approximated by neural networks, establishing that any\footnote{Any
    continuous or appropriately smooth function, depending on the
    precise model and topology.}  function can be approximated by a
  large enough network, and studying the {\em size} (number of units)
  required to achieve good approximation \cite[e.g.][]{hornik1989multilayer,cybenko1989approximations,barron1993universal,pinkus1999approximation}.  However, we are
  not aware of prior work establishing what {\em norm} is required for
  approximation, which, based on our current understanding of deep
  learning, is arguably the more important question (See Appendix \ref{app:barron} for a detailed comparison to \citet{barron1993universal}).
  
  There is a
  significant difference in studying approximation in terms of
  size (number of units) vs norm: most smooth functions require
  unbounded size in order to be approximated arbitrarily well (after
  all, functions that can be represented with a bounded number of
  weights form a finite dimensional class and thus occupy only zero
  measure in the infinite dimensional function space). Therefore,
  in classical approximation results the size goes to infinity as the
  approximation error goes to zero, and results focus on how quickly the size goes to infinity, or on approximation up to finite error.  In contrast, as we shall see here, broad classes
  of smooth functions can be approximated arbitrarily well, and even perfectly represented, with a
  bounded norm---the norm need not increase for the approximation to
  improve.  Can we characterize this class of functions?

  In this paper we consider infinite-width ReLU networks where the
  overall Euclidean norm (sum of squares of all weights in the system)
  is controlled.  More specifically, we consider networks with a single hidden
  layer, consisting of an unbounded number of rectified linear units
  (ReLUs)---see Section \ref{sec:reluNetworks} for a precise
  definition.  Such two-layer infinite networks were studied in the
  context of generalization \citep{neyshabur2015norm} and optimization
  \citep{bach2017convex,chizat2018global} and are were shown by
  \cite{neyshabur2014search} to be equivalent to ``convex neural
  nets'' as studied by \citet{bengio2006convex}.  

  In Theorem \ref{thm:main} we show that for univariate
  functions $f:\R\to\R$, i.e.~with a single real-valued input, optimizing a network's parameters while
  controlling the overall Euclidean norm is
  equivalent to fitting a function by controlling:
  \begin{equation*}
    \max\left( \int \abs{f''(x)} \diff x, \abs{f'(-\infty)+f'(+\infty)} \right).
  \end{equation*}
  We further show that fitting data while minimizing this complexity yields to linear spline interpolation.  Interestingly, such linear splines are
  exactly the predictors recently studied by \cite{belkin2018overfitting} as a model for understanding interpolation learning, though they fell short of
  connecting such models to neural networks.  Our work thus closes the loop and establishes a concrete connection between their analysis and neural network learning.
  
  We see then that even for univariate functions, Euclidean norm regularization on the weights already gives rise to very rich and natural induced bias in function space, that is not at all obvious and perhaps even surprising.  We of course view this as a starting point for studying multivariate functions, and in Section \ref{sec:highdim} discuss how our techniques can be extended, and conjecture an answer related to the integral of the nuclear norm of the Hessian.

  Our derivation of the induced complexity can be seen as an
  application of Green's function of the second derivative, and in
  Section \ref{sec:greens} we elaborate on this view, and describe how
  fitting a two-layer network can be viewed as a method for solving a
  variational problem using Green's functions.

\notinsubmission{
  Finally, we study the effects of depth and show that for a specific
  infinitely wide ReLU architecture of depth $L$, the induced
  complexity corresponds to an $(2/L)$-norm in a certain space,
  instead of $1$-norm for depth-2 networks.  This also agrees
  with prior work \citep{gunasekar2018implicit} on deep convolutional {\em linear} networks,
  suggesting different and ``sparser'' complexity control for deeper
  networks. 
}


\section{Infinite Width ReLU Networks}\label{sec:reluNetworks}

We consider 2-layer networks, with a single hidden layer consisting of an unbounded number of rectified linear units (ReLUs), defined by:
\begin{equation}
	\func{\params}(\xvec) = \sum_{i=1}^k w^{(2)}_i \relu{\ip{\wvec^{(1)}_i}{\xvec} + b^{(1)}_i} 
	+ 
	b^{(2)}
\label{eq-relu_net}
\end{equation}
over $\xvec\in\R^d$, where $\wvec^{(1)}_i$ are the rows of $\wmat^{(1)}$ and 
\begin{multline}
    \label{eq:ThetaClass2}
     \params \in \ThetaClass_2 = 
    \left\{ 
        \params=
        \left(
            k,\wmat^{(1)},\bvec^{(1)},
            \wvec^{(2)},b^{(2)}
        \right) 
        ~ \right| \\         
        \left.
        k\in\N, \wmat^{(1)}\in\R^{k \times d} , \bvec^{(1)}\in\R^{k}, \wvec^{(2)}\in\R^k, \bvec^{(2)}\in\R
    \right\}
\end{multline}
We associate with each network the squared Euclidean norm of the non-bias weights\footnote{Removing the output unit bias term $b^{(2)}$ will not change any of the definitions or results, since it can be simulated at no cost by a unit with infinitely large $b^{(1)}_i$, infinitesimally small $w^{(2)}_i$, and $\wvec^{(1)}_i=0$.  
The unregularized bias terms $b^{(1)}_i$ in the hidden layer are important to our analysis, and removing them or regularizing them would substantially change the definitions and results.}:
\begin{equation}
\begin{split}
    \Eucost(\params) &= 
	\half\left( 
    	\norm{\wvec^{(2)}}_2^2 + \norm{W^{(1)}}_F^2 
    \right)
	= 
    \half 
     \sum_{i=1}^k
    \left( 
        (w^{(2)}_i)^2 + \norm{\wvec^{(1)}_i} _2^2
	\right)
\end{split}
\end{equation}
and consider the minimum \dnote{infimum?} norm required to implement a given function $f:\R^d\to\R$,
\begin{align}\label{eq:Rf}
\begin{split}
R\prn{f} = {\inf_{\params \in \ThetaClass_2}}~
\Eucost(\params) 
 ~\text{\st}~
\func{\params} = f
\end{split}
\end{align}
Following\footnote{\citeauthor{neyshabur2014search} do not consider an unregularized bias, although this does not change the essence of their arguments.  In Appendix \ref{app:neyshaburreproof} we replicate their result, explicitly allowing for an unregularized bias.} \citet[Theorem 1]{neyshabur2014search}, minimizing $\Eucost(\params)$ is equivalent to constraining the norm of the weights $\wvec^{(1)}_i$ of each hidden unit and minimizing the $\ell_1$ norm of the top layer. That is, we can express \eqref{eq:Rf} as:
\begin{equation}\label{eq:finitel1}
    R(f) = {\inf_{\params \in \ThetaClass_2}}~\norm{\wvec^{(2)}}_1~\quad\text{s.t.}~h_\params = f~,~ 
    \forall{i}: \norm{\wvec^{(1)}_i}_2 = 1
\end{equation}
The above definitions of $R(f)$ require $f$ to be exactly implementable by some finite-size ReLU network, and so $R(f)$ is finite only for piece-wise linear functions with a finite number of pieces.  
However, any continuous function can be approximated by increasing the number of units arbitrarily.  Since we are not concerned with the number of units, only the norm, our central object of study is thus the norm required to capture a function as the number of units increases.  
This is captured by:
\begin{align}
\begin{split}
\oR\prn{f} = \lim_{\epsilon \to 0} \left( 
{\inf_{\params \in \ThetaClass_2}}~
\Eucost(\params) 
 ~\text{\st}~
\norm{\func{\params} - f}_\infty \leq \epsilon \right)
\end{split}
\label{eq:oR}
\end{align}

\removed{
\natinote{Move this as comment elsewhere, either is a short comment or footnote when presenting the networks, or to the discussion Section}
\begin{claim}[Claim on output bias]
\begin{align}
\begin{split}
R\prn{f} = {\min_{\params \in \ThetaClass_2, c \in \R}}~
\Eucost(\params) 
 ~~\text{\st}~
\func{\params} + c = f
\end{split}
\end{align}

Sketch: consider unit $\{w^{(2)}_k = \frac1{\gamma}, w^{(1)}_k = \vec 0, b^{(1)}_k = \gamma c \}$, as $\gamma \to \infty$.
\end{claim}
}
To understand $\oR(f)$ observe that by \eqref{eq:finitel1}, the sub-level set $\subF_B = \left\{ f \middle| R(f)\leq B\right\}$ is a scaling of the symmetric convex hull of ReLUs, plus arbitrary constant functions:
\begin{equation}
\subF_B = 
B \cdot \conv \left\{ 
    \xvec\mapsto \pm [\ip{\wvec}{\xvec}+b]_+ 
    \middle| \wvec \in\Sd, b\in\R 
    \right\} 
+ \left\{ \xvec \mapsto b_0 \middle| b_0\in\R \right\}\, ,
\end{equation}
where $\Sd=\{ \wvec\in\R^d | \norm{\wvec}_2=1 \}$.  
The sub-level sets of $\oR(f)$ are the closures $\overline{\subF_B}$ of the above convex hulls, 
or in other words the set obtained by taking all expectations w.r.t~all possible distributions, 
as opposed to only finite convex combinations (i.e.~only expectations w.r.t.~uniform discrete measures). 
We therefore have an infinite dimensional $L_1$ regularized problem: \natinote{Charlie/Daniel: is this rigorous and clear enough?  Do we need to cite something or say some other magic words or take care of some condition?  Is the $\norm{\cdot}_\infty$ norm above fine or do we want to work with some other topology?} \pnote{any other norm would imply that we'd have to discuss non-continuous functions, correct?}
\begin{align}\label{eq:infl1}
\overline R\prn{f} = {\inf_{\alpha \in \AlphaClass, \intbias \in \R}}~
\norm{\alpha}_1 
 ~~\text{\st}~
h_{\alpha, \intbias} = f
\end{align}
where $\AlphaClass$ is the set of all signed measures on $\Sd \times \R$, $\norm{\alpha}_1=\int\!\diff\!\abs{\alpha}\removed{=\int \abs{\alpha(\wvec,b)}\diff\wvec\diff b}$, and
\begin{align}\label{eq:alpha}
\begin{split}
h_{\alpha, \intbias} 
= 
\int_{\mathbb S^{d-1} \times \R} \!\!\!\!\! \relu{\ip{\wvec}{\xvec} + b} \diff \alpha(\wvec,b) +\intbias.
\end{split}
\end{align}
Learning an unbounded width ReLU network $h_\params$ by fitting some loss functional $L(\cdot)$ while controlling the norm $\Eucost(\params)$ by minimizing 
\begin{equation}\label{eq:minC}
    \min_{\params\in\ThetaClass_2} L(\func{\params}) + \lambda \cdot \Eucost(\params)
\end{equation}
is thus equivalent to learning a function $f$ while controlling $\oR(f)$:
\begin{equation}\label{eq:minR}
    \min_{f:\R\to\R \natinote{ \in \text{in what class?}}} L(f) + \lambda \cdot \oR(f)
\end{equation}
\pnote{can we even express the class here in a meaningful way without characterizing functions for which $\oR<\infty$ first, which only happens later on?}
From \eqref{eq:infl1}, it is clear that $\oR(f)$, and so also \eqref{eq:minR} are convex (although \eqref{eq:minC} is not!).  
Furthermore, for any lower semi-continuous\footnote{Lower semi-continuous in its first argument.  This is required only to ensure the minimum of \eqref{eq:minC} is attained.} scalar loss function 
$\loss:\R\times\R\to\R$, 
if we consider the empirical loss $L(f)=\sum_{i=1}^N \loss(f(x_i);y_i)$ over $\samples$ points, \eqref{eq:minC} has a minimizer where $\alpha$ is supported on at most $\samples$ weight vectors, 
\ie~uses at most $N+1$ units \citep{rosset2007l1}.  
Thus, as long as the number of units is at least as large as the number of data points, we already have that \eqref{eq:minC} is equivalent to \eqref{eq:minR} (strictly speaking: any global minimum of \eqref{eq:minC} also minimizes \eqref{eq:minR}).  Thus \eqref{eq:minR} not only tells us what happens in some infinite limit, but also precisely describes what we are minimizing with a finite, but sufficiently large, number of units.

Our goal is thus to calculate $\oR(f)$, for any function $f:\R^d\rightarrow\R$, and in particular characterize when it is finite, i.e.~understand which function can be approximated arbitrarily well with bounded norm but unbounded width ReLU networks. 

\natinote{Charlie: can you make sure the above is rigorous?  What norm over functions / topology / function space should we use?  Does it matter?}

\section{One-dimensional Functions}
\label{sec-1d}

\removed{

\begin{clemma}[Parameter alignment]
For any $2$-layer network $h_\Theta$ with $1$-dimensional inputs and weights $w^{(1)}, w^{(2)} \in \R^k$, if $\norm{h_\Theta } = R(h_\Theta)$, then:
\begin{equation*}
	\forall{i \in \iter{k}}: \quad |w^{(1)}_i| = |w^{(2)}_i|
\end{equation*}
\label{lemma-pal}
\end{clemma}
\begin{cproof}
    First note that we can rule out networks which, for some $i \in [k]$, exactly one of $w^{(2)}_i, w^{(1)}_i$ is zero. These networks can trivially have their norm decreased by having $w^{(2)}_i = w^{(1)}_i = 0$, which does not change its output.
    
    For the remaining networks, assume that the claim does not hold. For all units $i$ with non-zero weights, define $c_i = \sqrt{\frac{|w^{(2)}_i|}{|w^{(1)}_i|}}$, and consider parameters $\tilde \Theta$ given by, $\tilde w^{(1)}_i = c_i w^{(1)}_i $, $\tilde w^{(2)}_i = \frac{w^{(2)}_i}{c_i}$, $\tilde b^{(1)}_i = c_i b^{(1)}_i$. First, check that, for all such units:
    \begin{equation*}
    \begin{split}
	    \tilde w_i^{(2)}  
	    \relu{\tilde  w_i^{(1)} x + \tilde b_i^{(1)}} 
	    = \frac{w_i^{(2)}}{c_i} 
	    \relu{c_i(  w_i^{(1)} x +  b_i^{(1)})} 
	     = w_i^{(2)} \relu{w_i^{(1)} x + b_i^{(1)}}
	\end{split}
    \end{equation*}
    since $c_i > 0$ and $\relu{\cdot}$ is 
    $1$-positive-homogeneous. 
    Therefore $h_{\tilde \Theta} = h_\Theta$. 
    Moreover, we have, 
    for all units with non-zero weights:
    \begin{equation*}
	    (\tilde w^{(2)}_i)^2 +  (\tilde w^{(1)}_i)^2 = 2 |w^{(1)}_i w^{(2)}_i| \leq (w^{(2)}_i)^2 +  (w^{(1)}_i)^2
    \end{equation*}
    Since $|w^{(1)}_i| \neq |w^{(2)}_i|$ for some unit $i$ by assumption, the inequality will be strict for that unit. 
    Therefore $\norm{h_{\tilde \Theta}} < \norm{h_\Theta}$, contradicting $\norm{h_\Theta} = R(h_\Theta)$.
\end{cproof}
}

Our main result is an exact and complete characterization of $\oR(f)$ for univariate functions:

\begin{thm}\label{thm:main}
For any  $f: \R \to \R$, we have:
\begin{equation*}
	\oR(f) = \max \left( \, \int_{-\infty}^\infty \!\!\!\!\!\abs{f''(x)} \diff x  \, , \, \abs{f'(-\infty) + f'(\infty)} \, \right) 
	\, \leq  \int_{-\infty}^\infty \!\!\!\!\!\abs{f''(x)} \diff x + 2 \inf_x \abs{f'(x)}
\end{equation*}
\removed{where $\tv(f) = \int_{-\infty}^\infty \abs{f''(x)} \diff x$ if $f \in C^2$, and $\tv(f) = \sum_{j=1}^k \abs{l_j - l_{j-1}}$ if $f$ is PWL.}
\end{thm}
where $f''$ is the weak (distributional) 2nd derivative \removed{\citep{evans1991measure}\natinote{Does this book define distributions and weak derivatives??}} and so the integral is well defined even if $f$ is not differentiable, and is equal to the total variation of $f'$.   We also denote $f'(\infty)=\lim_{x\to\infty}f'(x)$ and $f'(-\infty)=\lim_{x\to-\infty}f'(x)$, noting that if $\int_{-\infty}^\infty \!\!\abs{f''(x)} \diff x$ is finite, both limits must exist.   The inequality and final expression are interpretable when $f$ is differentiable (though this can be relaxed).
\natinote{Is this rigorous, even if $f$ is not differentiable?}
In \thmref{thm:main}, its proof, and throughout, we work with distributions, 
rather than only functions, and as is common, we slightly abuse notation and use the same symbols to refer to both a function or measure and its associated distribution.

\begin{proof}
     In one dimension, we have that $w \in \Sd = \{\pm 1\}$ and it will be convenient for us to reparametrize the ReLUs as $\relu{w(x-b)}$  instead of $\relu{wx + b}$ (by transforming $(w,b) \mapsto (w,-wb)$, which does not change $\oR(f)$). In this form, $b$ exactly captures the threshold where a unit with parameters $(w,b)$ activates ($w = +1$) or deactivates ($w = -1$). For any representation $f=h_{\alpha,c}$ of this form we have:
    \begin{align}\label{eq:conv}
    f(x) &= \int_{\R} \left( \alpha(1,b) \relu{x - b}  + \alpha(-1,b) \relu{b - x} \right) \diff b +\intbias
    \end{align}
     where we are treating the measures over $b$ as distributions.  Taking the derivative twice w.r.t. $x$:
    \begin{align}
    f'(x) &= \int_\R \left( \alpha(1,b) H(x-b) - \alpha(-1,b) H(b-x) \right) \diff b \label{eq:fprime}\\
    f''(x) &= \int_\R \left( \alpha(1,b) + \alpha(-1,b) \right) \delta_x(b) \diff b  \notag 
    \\
    &= \alpha(1,x) + \alpha(-1,x) = \alpha_+(x) \label{eq:key}
    \end{align}
    where $H(z)$ is the Heaviside step function ($H(z)=1$ when $z>0$ and zero otherwise), whose distributional derivative is the dirac distribution $\delta_a(z)$, which assigns point-mass to $z=a$, and we defined $\alpha_+(b)=\alpha(1,b)+\alpha(-1,b)$.
    
    We see that the measure $\alpha$ that represents a function $f$ is \textit{almost} unique: the component $\alpha_+$ corresponding to the total (signed) mass 
    is unique and determined precisely the second derivative of $f$. 
    The only flexibility is in shifting mass between the forward and backward sloping ReLUs with threshold $b$, \ie~between $\alpha(1,b)$ and $\alpha(-1,b)$.  
    We denote this component by $\alpha_-(b)=\alpha(1,b)-\alpha(-1,b)$, which together with $\alpha_+$ defines $\alpha$ as $\alpha(w,b)=\half(\alpha_+(b)+w\cdot\alpha_-(b))$.  
    As the following calculation shows the component $\alpha_-$ only contributes an affine component to $f=h_{\alpha,c}$:
    \begin{align}
    f(x) &= \int_{\R} \left( \alpha(1,b) \relu{x - b}  + \alpha(-1,b) \relu{b - x} \right) \diff b + \intbias  \\*
    \removed{ & = \int_{\R} \frac{\abs{x - b} + (x-b)}2 \diff \alpha(1,b) + \int_{\R} \frac{\abs{b-x} + (b-x)}2 \diff \alpha(-1,b) \diff b + \intbias \\}
    &= \half \int_{\R} \alpha_+(b) \abs{x - b} \diff b + \half \int_{\R} \alpha_-(b) (x-b)\diff b + \intbias \\*
    &= \half  \int_\R f''(b) \abs{x-b} \diff b + \left( \half \int_{\R} \alpha_-(b)\diff b\right) x  + \left( -\int_{\R} b \alpha_-(b) \diff b + \intbias \right) \label{eq:affine}
    \end{align}
    \removed{
    where we used the fact that $\relu{x} = \frac{\abs{x} + x}2$. Now, let $\diff \alpha_+(b) = \diff \alpha(1,b) + \diff \alpha(-1,b)$ and $\diff \alpha_-(b) = \diff \alpha(1,b) - \diff \alpha(-1,b)$. We get:
    \begin{align}
    \begin{split}
    h_{\alpha, \intbias}(x) &= \intbias + \frac12 \int_{\R} \abs{x - b} \diff \alpha_+(b) + (x-b) \diff \alpha_-(b) \\
    & = \intbias + \frac12 \left( \int_{\R} \abs{x - b} \diff \alpha_+(b) + x \int_{\R} \diff \alpha_-(b) - \int_{\R} b \diff \alpha_-(b) \right)
    \end{split}
    \label{eq-alphap-alpham}
    \end{align}
    First, we will show that, if $f = h_{\alpha, \intbias}$, then $\alpha_+$ is unique and given by $\alpha_+(b) = f''(b)$. By differentiating Equation~\eqref{eq-alphap-alpham} twice w.r.t. $x$:
    \begin{align}
    \begin{split}
    h'_{\alpha, \intbias}(x) &= \frac12 \int_{\R} \sign\left({x - b}\right) \diff \alpha_+(b) + \diff \alpha_-(b)
    \end{split}
    \label{eq-int_diff1}
    \end{align}
    \begin{align}
    \begin{split}
    h''_{\alpha, \intbias}(x) &= \int_{\R} \delta(x - b) \diff \alpha_+(b) = \alpha_+(x)
    \end{split}
    \label{eq-int_diff2}
    \end{align}
    where $\delta(x-b)$ assigns point-mass to $b=x$, and zero-mass to $b\neq x$. Hence, if $f = h_{\alpha, \intbias}$, we have that $f''(x) = \alpha_+(x)$. 
    On the other hand, $\alpha_-$ captures a linear term of $h_{\alpha, \intbias}$ and is not unique (more specifically, 
    the term $x \int_{\R} \diff \alpha_-(b)$ in Equation~\eqref{eq-alphap-alpham}, 
    since $\int_{\R} b \diff \alpha_-(b)$ can always be compensated by an appropriate choice of $\intbias$). To see this, consider:
    }
    Our only constraint in choosing $\alpha_-$ is thus in getting the correct linear term in \eqref{eq:affine}, as we can always adjust the constant term using the bias $\intbias$ without affecting $\oR(f)$.  
    To understand what this linear correction must be, 
    we can use \eqref{eq:fprime} to evaluate $f'(-\infty)$ and $f'(+\infty)$ (note that if $\int \abs{f''} \diff x = \int \abs{\alpha_+(b)} \diff b \leq \int \abs{\diff \alpha}$ is finite, then $f'$ must converge at $\pm\infty$ \natinote{Charlie---is this statement fine as is?}) and we get:
    \begin{align}
    f'(-\infty) + f'(+\infty) 
    = \int_\R (0-\alpha(-1,b))\diff b + \int_\R (\alpha(1,b)-0)\diff b
    = \int_\R \alpha_-(b) \diff b
\removed{    
    \frac12 \left( \int_{\R} \left( \diff \alpha_-(b) -\diff \alpha_+(b) \right) + \int_{\R} \left(\diff \alpha_-(b) + \diff \alpha_+(b) \right) \right) = \int_{\R} \diff \alpha_-(b)              }
    \label{eq-const_alpham}
    \end{align}
    Any measure $\alpha_-$ that integrates as \eqref{eq-const_alpham}, in conjunction with $\alpha_+=f''$ and an appropriate $\intbias$, will yield $f=h_{\alpha,\intbias}$ with
    \begin{equation}\label{eq:objalphaminus}
        \norm{\alpha}_1 = \frac12 \int_\R \left( \abs{f''(b) + \alpha_-(b)}  +  \abs{f''(b) - \alpha_-(b)}\right) \diff b.
    \end{equation}
    To minimize $\norm{\alpha}_1$ we must therefore solve the following convex program:
    \begin{align}\label{eq:alphaminusproblem}
    \begin{split}
        \min_{\alpha_-} &\frac12 \int_\R \left( \abs{f''(b) + \alpha_-(b)}  +  \abs{f''(b) - \alpha_-(b)}\right) \diff b\\
        \textrm{s.t.} &\quad \int_{\R} \alpha_-(b) \diff b = f'(-\infty) + f'(\infty)
    \end{split}
    \end{align}
    Introducing the Lagrange multiplier $\lambda\in\R$, and setting the derivative of the Lagrangian $\mathcal L$ w.r.t.~$\alpha_-$ to zero we have:
    \begin{align}\label{eq:L0}
    \begin{split}
        0 \in \frac{\partial \mathcal L}{\partial \alpha_-} = \half\left( \sign\prn{f'' + \alpha_-} - \sign\prn{f'' - \alpha_-} + 2\lambda\right)
    \end{split}
    \end{align}
    Consider the possible values $\lambda$ might take:
    \begin{description}
        \item[$\lambda = 0$]:
            We have $\sign\prn{f'' + \alpha_-} 
            = 
            \sign\prn{f'' - \alpha_-}$, and hence $\abs{\alpha_-} \leq \abs{f''}$ pointwise,  
            and so from \eqref{eq:objalphaminus} we can calculate $\norm{\alpha}_1=\int_\R \abs{f''(b)} \diff b$.  
            For the constraint in \eqref{eq:alphaminusproblem} to hold, we have:
            \begin{align}
    \begin{split}
        \abs{f'(-\infty) + f'(\infty)} = \abs{\int_{\R} \alpha_-(b) \diff b} \leq \int_\R \abs{\alpha_-(b)}\diff b \leq \int_{\R} \abs{f''(b)} \diff b. 
    \end{split}
    \end{align}
        \item[$\lambda<0$]:
        For \eqref{eq:L0} to hold with $\lambda<0$ we must have $f''+ \alpha_- \geq 0$ and $f'' - \alpha_- \leq 0$ pointwise, hence  $\alpha_- \geq \abs{f''}$ and from \eqref{eq:objalphaminus} and the constraint in \eqref{eq:alphaminusproblem}:  $\norm{\alpha}_1 = \int_\R \alpha_-(b) \diff b = f'(-\infty) + f'(\infty)$.  This case is possible if we can make the constraint hold, \ie~if and only if $f'(-\infty) + f'(\infty) \geq \int_{\R} \abs{f''(b)} \diff b$
        
        \item [$\lambda >0$]:
        Symmetric to $\lambda <0$. We get $\norm{\alpha}_1 = \int_{\R}(- \alpha_-)\diff b = -(f'(-\infty) + f'(\infty))$ and this happens when $f'(-\infty) + f'(\infty) \leq -\int_{\R} \abs{f''(b)} \diff b$.
    \end{description}
    Combining the cases above, we have $\norm{\alpha}_1=\max( \int \abs{f''}\diff b , \abs{f'(-\infty)+f'(+\infty)} )$.  To get the inequality in the Theorem statement, note that for any $x$: 
    \begin{align}
    \abs{f'(-\infty)+f'(+\infty)}&=\abs{\left(f'(x)-\int_{-\infty}^x\!\!\!\!\! f''(b)\diff b \right) + \left( f'(x)+\int_{x}^\infty \!\!\!\!\! f''(b) \diff b \right) } \\
    &\leq 2 \abs{f'(x)}+\int_{\infty}^\infty \abs{f''}\diff b. \notag   \qedhere
    \end{align}    
\end{proof}
\begin{cor}
For any loss function $L(f)$ over $f:\R\to\R$, 
fitting a regularized infinite width ReLU network by minimizing $\argmin_{\theta\in\ThetaClass_2}L(h_\theta)+\lambda \cdot C(\theta)$ is equivalent to\dnote{f should be in some space here, no?}\natinote{Presumably, yes.  Not sure how to state this.  Its over "any" distribution, so I guess the space is defined by the test functions?  Or in terms of functions, its essentially those where the TV of the $f'$ is bounded... Do you know what the correct space $f$ should live in?}\natinote{On second thought, the Theorem holds for $f:R\rightarrow\R$.  The proof goes through distributions, but we still need only consider actual functions}:
\begin{equation}\label{eq:cor}
    \argmin_{f:\R\rightarrow\R \removed{: \oR(f) < \infty\natinote{This is already implied by the constraint and doesn't actually specify the space we work in}}} L(f) + \lambda \cdot \max \left( \, \int_{-\infty}^\infty \!\!\!\!\!\abs{f''(x)} \diff x  \, , \, \abs{f'(-\infty) + f'(\infty)} \, \right)
\end{equation}
\label{cor:main}
\end{cor}
\corref{cor:main} shows the learning with norm-regularized ReLU networks is essentially equivalent to learning by minimizing the total variation of the derivative.  How does fitting data while minimizing this total variation look like?  What kind of functions would we get? Consider first  perfectly fitting (interpolating) a finite set of points $S = (x_{\smplIdx},y_{\smplIdx})_{\smplIdx=1}^\samples,\,\, x_{\smplIdx},y_{\smplIdx} \in \R$, given by:
\begin{multline}
     \argmin_{h_\params\,:\,\params \in \ThetaClass_2} C(\theta) \,\, s.t. \,\left( \forall_{\smplIdx \in [\samples]}, \func{\params}(x_\smplIdx) = y_\smplIdx  \right) \quad = \quad 
     \argmin_{f : \R \to \R} \oR(f) \,\, s.t. \, \left(\forall_{\smplIdx \in [\samples]}, f(x_\smplIdx) = y_\smplIdx\right)
    \label{eq:finiteobj}
\end{multline}

\begin{thm}
For any dataset $S = (x_{\smplIdx},y_{\smplIdx})_{\smplIdx=1}^\samples,\,\, x_{\smplIdx},y_{\smplIdx} \in \R$, 
where w.l.o.g.~$x_1 < x_2 < \cdots < x_{\samples}$, \eqref{eq:finiteobj} is minimized by the linear spline interpolation:
\begin{equation}
    \pwlf(x) = 
    \begin{cases}
y_1 + l_0(x - x_1) & x \leq x_1
\\
 y_\smplIdx + l_\smplIdx (x - x_\smplIdx) & x_\smplIdx \leq x \leq x_{\smplIdx+1} \quad \text{ for some } \smplIdx \in [\samples]
 \\
y_\samples + l_\samples(x - x_\samples) & x \geq x_{\samples}
\end{cases}
\label{eq:spline}
\end{equation}
where $l_\smplIdx = \frac{y_{\smplIdx+1} - y_\smplIdx}{x_{\smplIdx+1} - x_\smplIdx}$ for $n \in \{1, \cdots, \samples-1\}$, and $l_0,l_N$ are chosen so as to minimize
\begin{equation}
\begin{split}
 \max \left( \sum_{\smplIdx=0}^{\samples-1} \abs{l_{\smplIdx+1} - l_\smplIdx} , \abs{l_0 + l_N}  \right)
\label{eq:l0ln}
\end{split}
\end{equation}
\label{thm:spline}
\end{thm}

\begin{proof}
\removed{Let w.l.o.g. $\optf$ be a continuously differentiable function 
\footnote{For any minimizer $\optf$ of \eqref{eq:cor}, $\optf'$ must have bounded variation (otherwise $\int \abs{\optf''(x)} \diff x = \infty$), and hence $\optf$ can be approximated by a continuously differentiable function $g$ such that $\int \abs{g''(x)} \diff x \leq \int \abs{\optf''(x)} \diff x$ \citep{evans1991measure} {\color{red}quite handwavy, go back to this later}}
that minimizes \eqref{eq:variational-obj}. It is easy to see that $\pwlf$ satisfies all its constraints, i.e. $\pwlf(x_\smplIdx) = y_\smplIdx$ for all $\smplIdx \in [\samples]$. Moreover, we will show that $\oR(\pwlf) \leq \oR(\optf)$, hence $\pwlf$ is also a minimizer of \eqref{eq:variational-obj}.

First, note that $\pwlf''(x) = 0$ for all $x \in \R$ except for its $\samples$ breakpoints $\{x_\smplIdx\}_{\smplIdx=1}^\samples$. Additionally, for any interval $[a,b]$ that contains a single breakpoint $x_\smplIdx$, we get, for $x \in [a,b]$, $\pwlf''(x) = (l_\smplIdx - l_{\smplIdx-1}) \delta_{x_\smplIdx}(x)$  (to see this, take the second derivative of $g(x) = \pwlf(x_\smplIdx) + l_\smplIdx \relu{x-x_\smplIdx} - l_{\smplIdx-1} \relu{x_\smplIdx-x}$, which matches $\pwlf(x)$ in the interval $[a,b]$). Therefore, we get:}
First, verify that $\int_{\R} \abs{\pwlf''(x)} \diff x = \sum_{\smplIdx=0}^{\samples-1} \abs{l_{\smplIdx+1} - l_\smplIdx}$ and so $\oR(\pwlf)$ is given by \eqref{eq:l0ln}.  We now need to show that \eqref{eq:l0ln} is a lower bound on the value in \eqref{eq:finiteobj}.  Following \cite{rosset2007l1}, and when we consider \eqref{eq:finiteobj} as minimization w.r.t.~the measure $\alpha$ and $c\in\R$, it is always minimized by $\optf=h_{\alpha,c}$ where $\alpha$ is discrete with support of size at most $\samples+1$, corresponding to a piece-wise linear $\optf$ with at most $\samples+1$ breakpoints.  But instead of working with piece-wise linear functions, we will rely on the fact that by smoothing $\optf$ we can always approach it with a sequence of twice continuously differentiable functions{\footnote{Take $f_r$ to be a convolution of $f^*$ with $\exp(-rx^2)$, so that $\norm{\optf-f}_\infty\rightarrow 0$ and $\norm{f'_*-f_r'}_1\rightarrow 0$.\natinote{Would be good to add a sentence about why $R$ converges.}}} $f_r\rightarrow \optf$, with $\oR(f_r)\rightarrow \oR(\optf)$.  It is thus sufficient to prove that $\lim\inf \oR(f_r) \geq \oR(\pwlf)$.  For $f_r$ s.t. $\norm{\optf-f_r}<\epsilon$, since $f_r$ is continuously differentiable, for all $n=1,\ldots,N-1$, there exists a midpoint $x_\smplIdx \leq z_\smplIdx \leq x_{\smplIdx+1}$ such that $\abs{f'_r(z_\smplIdx) - l_\smplIdx } \leq 2\epsilon / \delta $ where $\delta = \min_n (x_{n+1}-x_n)$.  Recalling that $\oR(\pwlf)$ is the minimum of \eqref{eq:l0ln} over $l_0,l_N$, and plugging in $l_0=f'_r(-\infty)$ and $l_N=f'_r(+\infty)$, we have that $\oR(f_r)> \oR(\pwlf)-4 N \epsilon/\delta$.  Taking $r\rightarrow\infty$ and so $\epsilon\rightarrow 0$, we have $\oR(\optf) = \lim_{r\rightarrow\infty} \oR(f_r) \geq \oR(\pwlf)$, which establishes that $\pwlf$ minimizes \eqref{eq:finiteobj}.
\end{proof}

\begin{figure}[t]
\centering 
    \begin{tikzpicture}[spy using outlines={rectangle,yellow,magnification=3.5,width=4.5cm, height=1.5cm, connect spies}]
    \node {\pgfimage[width=5cm]{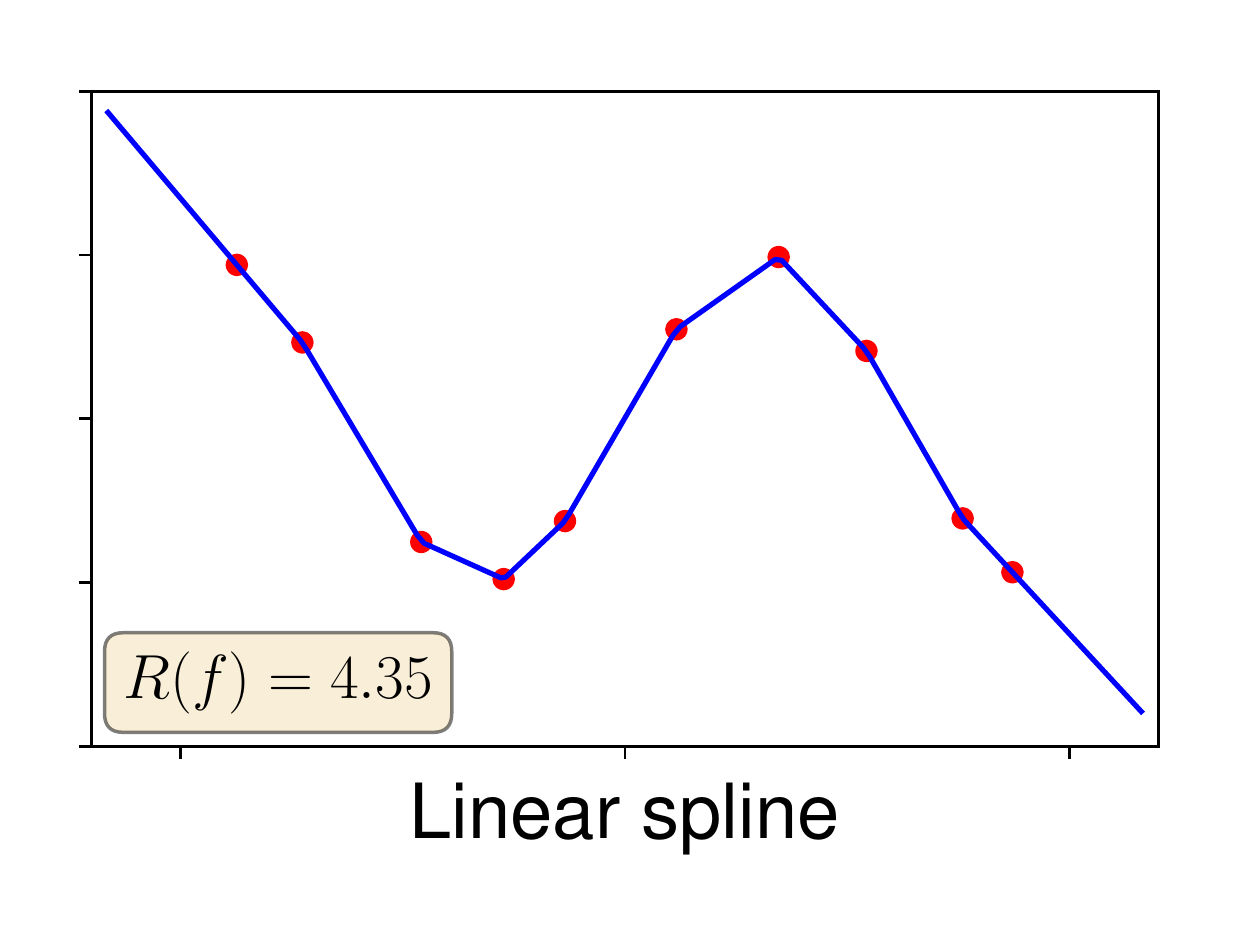}};
    \spy on (-0.5,-0.33) in node [left] at (2,2);
    \end{tikzpicture}
    \hspace{-2em}
    \begin{tikzpicture}[spy using outlines={rectangle,yellow,magnification=3.5,width=4.5cm, height=1.5cm, connect spies}]
    \node {\pgfimage[width=5cm]{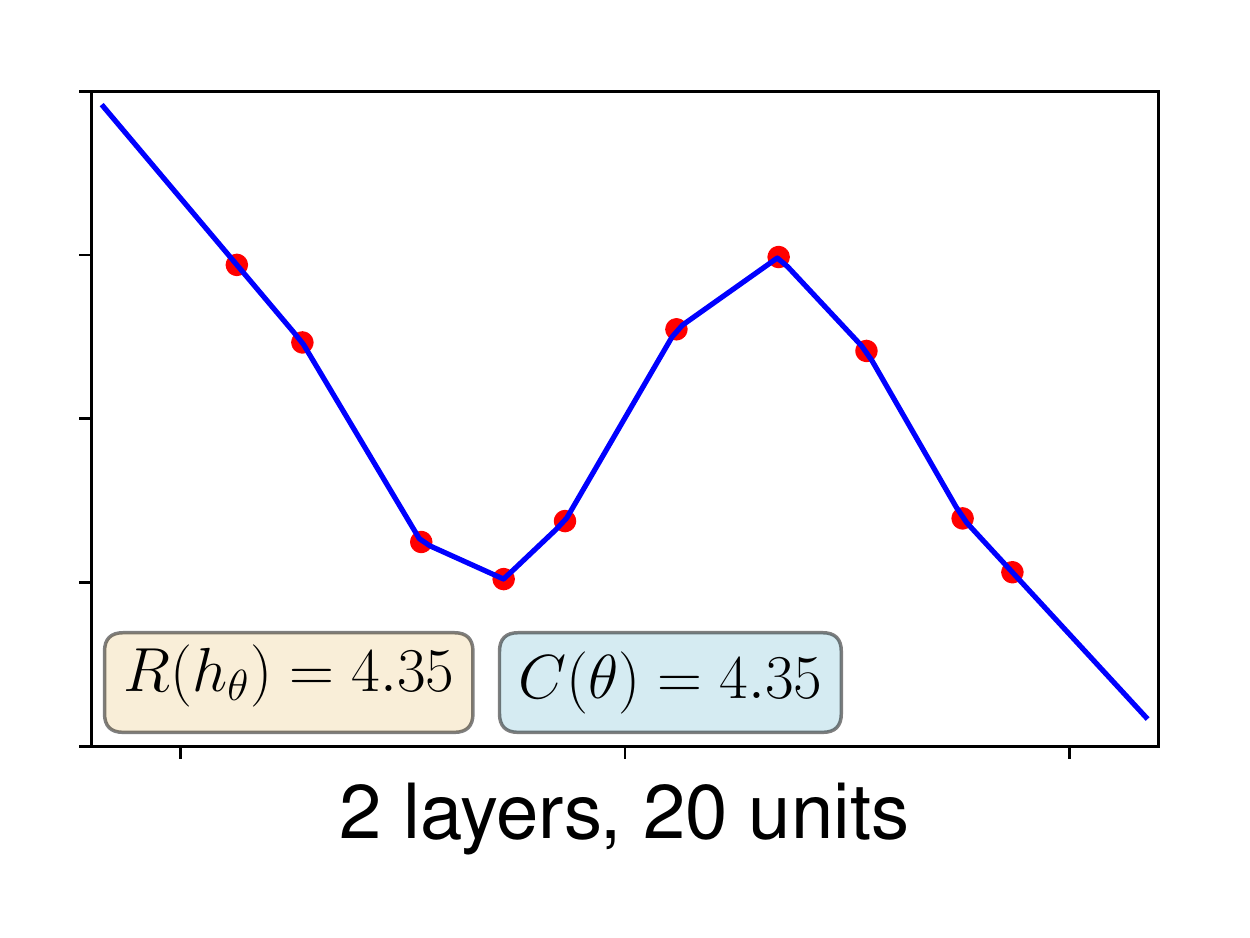}};
    \spy on (-0.5,-0.33) in node [left] at (2,2);
    \end{tikzpicture}
    \hspace{-2em}
    \begin{tikzpicture}[spy using outlines={rectangle,yellow,magnification=3.5,width=4.5cm, height=1.5cm, connect spies}]
    \node {\pgfimage[width=5cm]{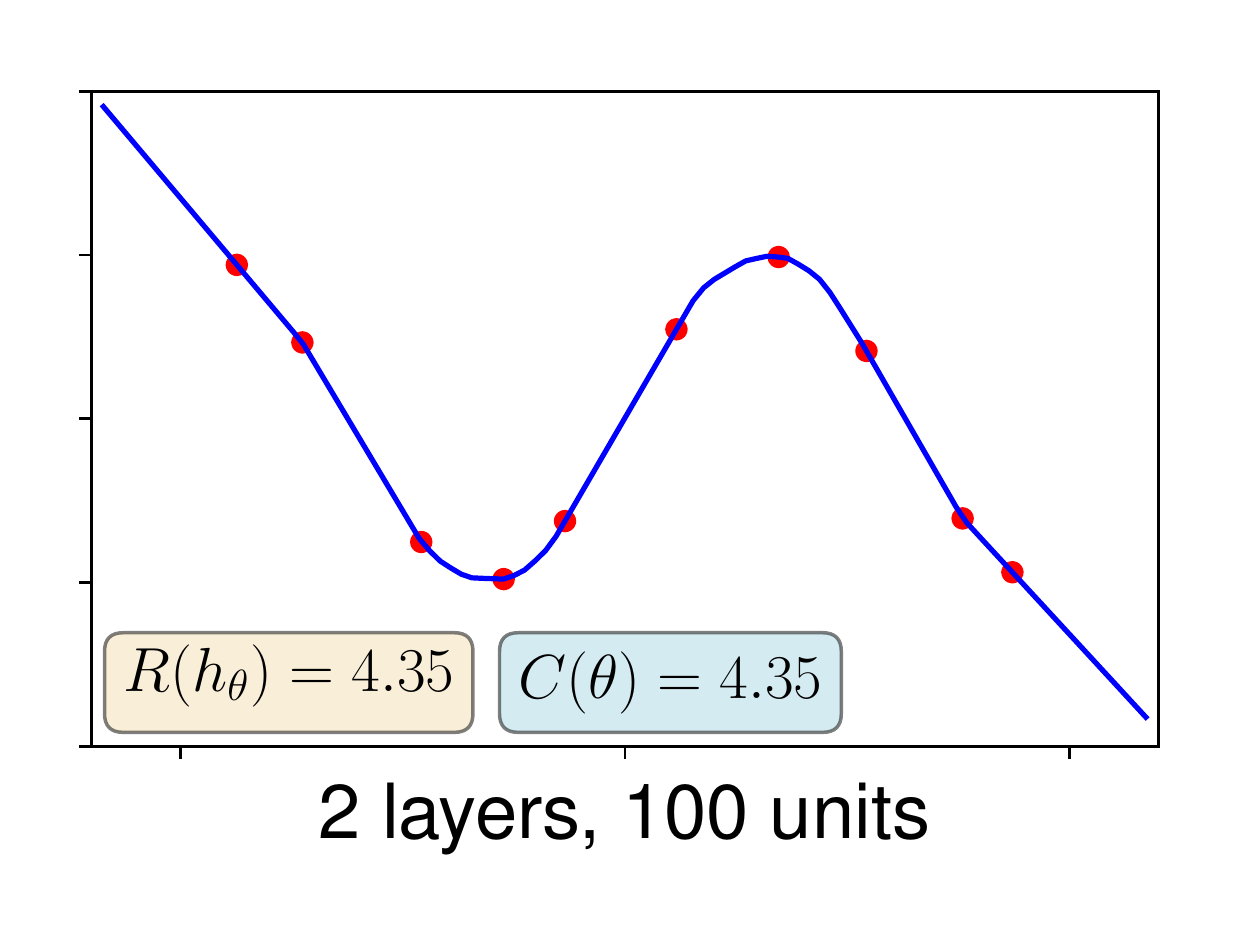}};
    \spy on (-0.5,-0.33) in node [left] at (2,2);
    \end{tikzpicture}
\caption{Linear interpolation (\textbf{left}) and two trained ReLU networks with 1 hidden layer consisting of 20 (\textbf{middle}) and 100 (\textbf{right}) units respectively, optimized to perfectly fit a set of 10 points, for which the minimum cost of perfect fitting is $\oR(\optf) = 4.35$. Training was done by minimizing the squared loss with a small regularization of $\lambda = 10^{-5}$. All three functions achieve the optimal cost $\oR(\cdot)$ in function space, and both networks yield optimal cost in parameter space $C(\theta)$.  The two networks arrived at different global minima in function space, with the same value of $\oR(f)$. For example, in the area highlighted, changing the derivative gradually instead of abruptly does not effect its total variation, and so also yields an optimal solution. \label{fig:2layernets}}
\end{figure}

\removed{
\begin{proof} (cleaner attempt, directly against PWL {\color{red} not finished, looks messier, go with above one probably})

    From \citep{rosset2007l1}, we know that if \eqref{eq:regularized-obj} is attainable, then there exists a minimizer $\params^*$ such that $\optf = \func{\params^*}$ is a continuous piece-wise linear function with at most $N+1$ pieces, like $\pwlf$. Hence, the only flexibility is on the placement of the breakpoints. It then suffices to show that having breakpoints exactly at the $\samples$ datapoints $\{x_\smplIdx\}_{\smplIdx=1}^\samples$ is indeed optimal.
    
    First, note that since $\optf$ fits every data point perfectly, then it must be at least $\abs{l_\smplIdx}$-Lipschitz in each $[x_\smplIdx, x_{\smplIdx+1}]$ interval $\left(\abs{l_\smplIdx} = \abs{ \frac{y_{\smplIdx+1} - y_\smplIdx}{x_{\smplIdx+1} - x_\smplIdx}}\right)$, hence there must exist a point $z_\smplIdx \in [x_\smplIdx, x_{\smplIdx+1}]$ such that $\abs{\optf'(z_\smplIdx)} \geq \abs{l_\smplIdx}$, as $\pwlf$ is almost everywhere differentiable \pnote{maybe worth to get into detail, i.e. can break the interval into finitely many subintervals where $\optf$ is everywhere differentiable in each subinterval, and it must be at least $\abs{l_\smplIdx}$-Lipschitz in some subinterval}. With this, we get:
\begin{equation}
\begin{split}
    \int_{\R} \abs{{\optf}''(x)} \diff x 
    &= 
    \int_{-\infty}^{z_1} \abs{\optf''(x)} \diff x +
     \sum_{\smplIdx=1}^{\samples-2} \int_{z_\smplIdx}^{z_{\smplIdx+1}} \abs{{\optf}''(x)} \diff x +
     \int_{z_{\samples-1}}^{\infty} \abs{\optf''(x)} \diff x
\end{split}
\end{equation}
Ultra sketch: know that some $N+1$ piece PWL minimizes the objective (from rosset). for each interval $[x_n, x_{n+1}]$ must have a point $z_n$ s.t. $|f'(z_n)| \geq \abs{l_n}$ (easy to show), break integral over $\R$ into these points (plus $(-\infty, z_1)$, $(z_n, \infty)$) and show lower bound of $R(\pwlf)$, using the fact that $l_+, l_-$ minimize $R(\pwlf)$.
\end{proof}

}

Theorem \ref{thm:main} guarantees that the linear spline interpolation is {\em a} global minimum of \eqref{eq:finiteobj}, but it will in general not be unique, as demonstrated in Figure \ref{fig:2layernets}, and networks learned in practice might implement much ``smoother'' functions.  The same type of solutions will also be obtained when minimizing \textit{any} loss over a finite sample:

\begin{cor}[Optimality of Linear Interpolation]
For any dataset $S = (x_{\smplIdx},y_{\smplIdx})_{\smplIdx=1}^\samples,\,\, x_{\smplIdx},y_{\smplIdx} \in \R$, 
and any lower semi-continuous loss 
$\loss : \R \times \R \to \R\cup\{\infty\}$ and any $\lambda>0$, consider:
\begin{equation}
    \argmin_{h_\theta : \theta \in \ThetaClass_2} \sum_{\smplIdx=1}^{\samples} 
    \loss (h_\theta(x_{\smplIdx}), y_{\smplIdx}) + \lambda \cdot C(\theta).
\label{eq:argmin}
\end{equation}
Then \eqref{eq:argmin} will always have a global minimum which is a piece-wise linear function with at most $\samples+1$ pieces, 
with breakpoints at the data points $x_1,\ldots,x_{\samples}$.  
\label{cor:interp}
\end{cor}
\begin{proof}
A minimizer $h_{\theta^*}$ of \eqref{eq:argmin} is also a minimizer of \eqref{eq:finiteobj} with the alternative labels $y_n=h_{\theta^*}(x_n)$.
\end{proof}

\begin{figure}
\centering 
    \begin{tikzpicture}[spy using outlines={rectangle,yellow,magnification=3.5,width=4.5cm, height=1.5cm, connect spies}]
    \node {\pgfimage[width=5cm]{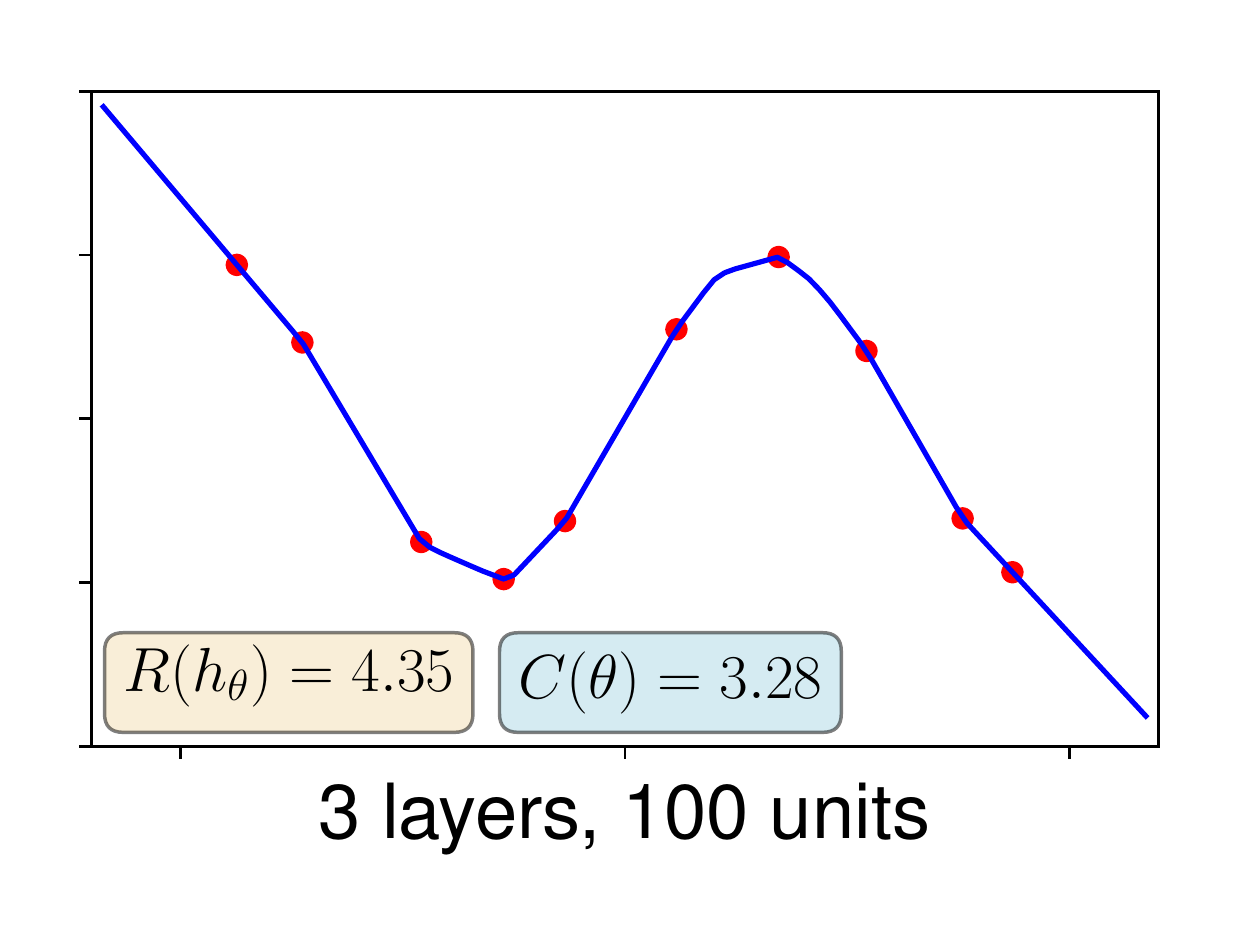}};
    \spy on (-0.5,-0.33) in node [left] at (2,2);
    \end{tikzpicture}
    \hspace{-2em}
    \begin{tikzpicture}[spy using outlines={rectangle,yellow,magnification=3.5,width=4.5cm, height=1.5cm, connect spies}]
    \node {\pgfimage[width=5cm]{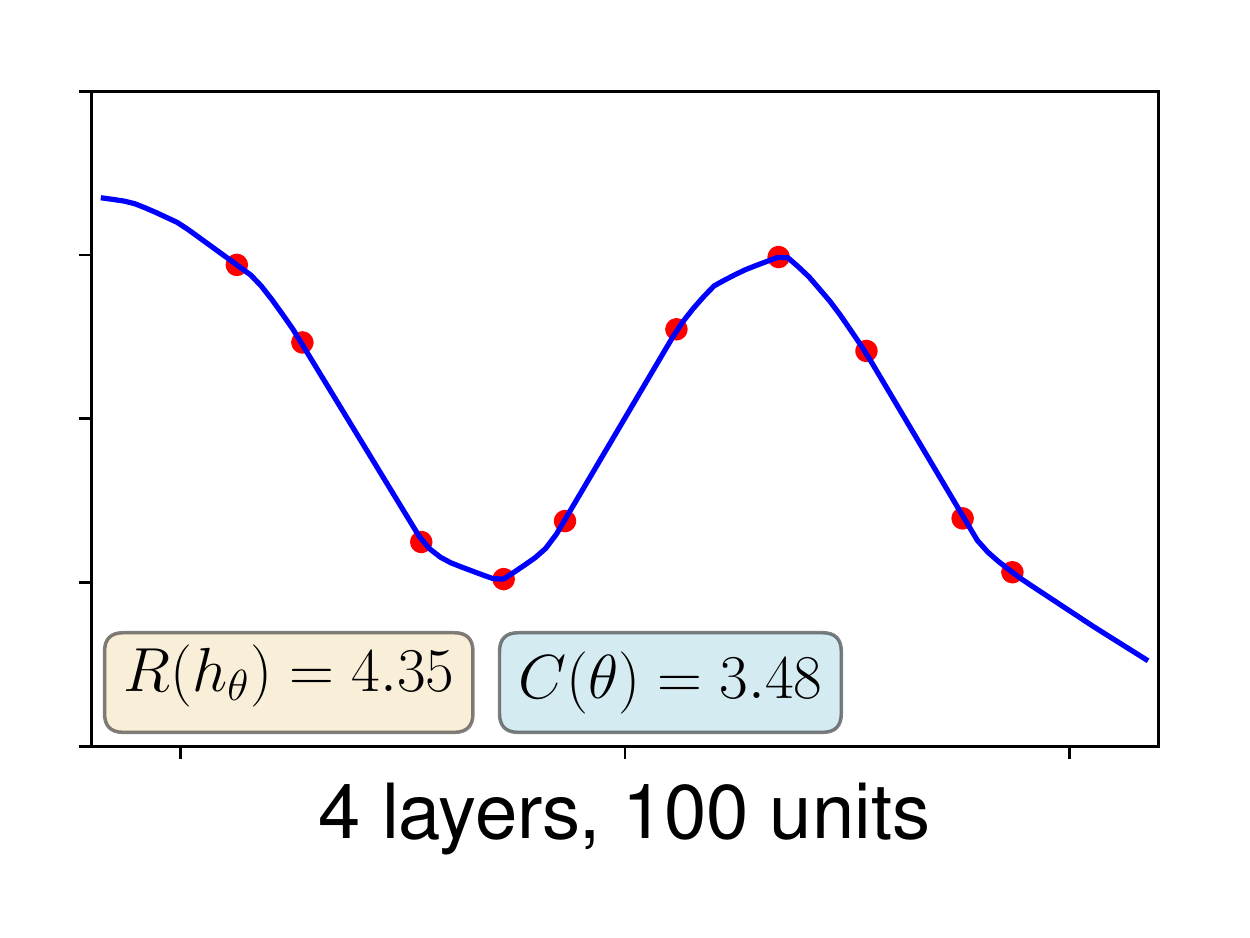}};
    \spy on (-0.5,-0.33) in node [left] at (2,2);
    \end{tikzpicture}
    \hspace{-2em}
    \begin{tikzpicture}[spy using outlines={rectangle,yellow,magnification=3.5,width=4.5cm, height=1.5cm, connect spies}]
    \node {\pgfimage[width=5cm]{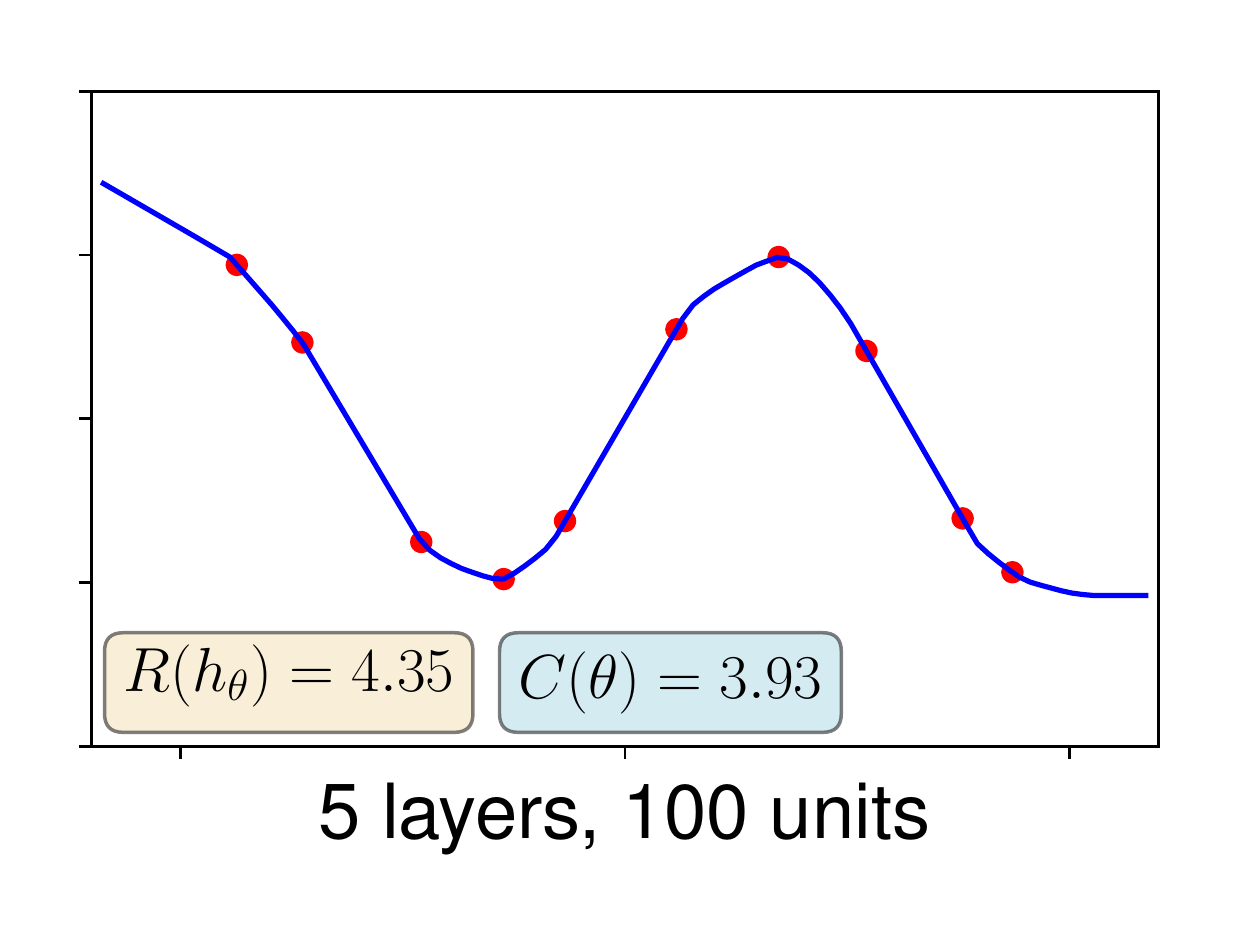}};
    \spy on (-0.5,-0.33) in node [left] at (2,2);
    \end{tikzpicture}
\caption{Deeper ReLU networks with 2 (\textbf{left}), 3 (\textbf{middle}) and 4 (\textbf{right}) hidden layers, each with 100 hidden units, trained on the same set of points as the 1-hidden layer networks on \figref{fig:2layernets}, also with $\lambda = 10^{-5}$. \label{fig:deepernets}}
\end{figure} 
\section{An Interpertation in Terms of Green's Functions}
\label{sec:greens}

The key component in the proof of \thmref{thm:main} is in writing a two-layer network as a convolution with ReLU functions as in \eqref{eq:conv}, and taking the second derivative of this convolution (in Equation~\eqref{eq:key}), noticing that the ReLU is a Green's function for the second derivative.  Fitting a ReLU network can therefore be seen as the reverse of using Green's function to solve a differential equation: we know $f:\R\to\R$ and would like to calculate it's 2nd derivative $u=f''$.  We can do this by fitting an infinite width ReLU network and then reading off the second derivative from the resulting weights.

In terms of the representation \eqref{eq:conv}, we get the second derivative directly from the measure $\alpha$, as $u(x)=\alpha_+(x)=\alpha(1,x)+\alpha(-1,x)$.  For the more standard integral representation \eqref{eq:alpha}, we have $u(x)=\alpha(1,x)+\alpha(-1,-x)$.  
Fitting an actual network with discrete units and weights on each unit, as in \eqref{eq-relu_net}, corresponds to an approximation of the form:
\begin{equation}\label{eq:uapprox}
    u \approx \sum_i w^{(2)}_i \abs{w^{(1)}_i} \delta_{b/w^{(1)}_i} 
\end{equation}
where to get a smoother approximation we might want to replace the delta function with a smoother bump.  

\natinote{Can someone please check the above, that I got the right form in \eqref{eq:uapprox}?}

\natinote{Would be neat to see how \eqref{eq:uapprox} looks like numerically for some simple function.}

Ignoring for the moment the linear term, 
we saw in \corref{cor:main} that fitting a two-layer ReLU network while controlling the norm of the weights corresponds essentially to the variational problem:
\begin{equation}\label{eq:var}
    \min_f L(f) + \int_\R \abs{f''(x)}\diff x.
\end{equation}
An interpretation of how this is done using the ReLU network is as follows:  
we first introduce an auxiliary function variable $u$,  
and rewrite \eqref{eq:var} as:
\begin{equation}\label{eq:fandu}
    \min_{f,u} L(f) + \int_\R \abs{u(x)}\diff x \quad \textrm{s.t. } u=f''.
\end{equation}
We then use the fact that the ReLU is a Green's function for the second derivative to write
\begin{equation}\label{eq:fasuconv}
    f(x)=\int_\R u(b) [x-b]_+ \diff b + a x + c
\end{equation}
where $ax+c$, 
for some $a,x\in\R$, 
is the affine term left undetermined by the second derivative.  
Plugging \eqref{eq:fasuconv} back into \eqref{eq:fandu} we get:
\begin{equation}\label{eq:qvar}
    \min_u L \left(x\mapsto \int_\R u(b) [x-b]_+ \diff b + a x + c\right) + \int_\R \abs{u(b)}\diff b
\end{equation}
which amounts to fitting an infinite ReLU network with $L_1$ regularization on the top layer, which in turn, using the equivalence between \eqref{eq:finitel1} and \eqref{eq:Rf}\todo{\cite{neyshabur2015norm}}, 
is equivalent to $\Lp{2}$-regularization on both layers. 
The additional term $\abs{f'(-\infty),f'(+\infty)}$ in 
\thmref{thm:main} comes from the fact that unlike in the derivation above, we are also constraining the linear term.

\section{Approximating Higher Dimensional Functions}
\label{sec:highdim}
Next, we discuss how our results may generalize to higher dimensions, i.e.~for networks with multiple input units where the input $\xvec$ is in $\mathbb{R}^d$, and help us in characterizing $\oR(f)$ for general $f:\R^d\to\R$.
First, it easy to see that we can extend our Green's function view for higher dimensional inputs.  For any representation $f=h_{\alpha,c}$ as in \eqref{eq:alpha} we have:
\removed{
Specifically, for any function $\mathbb{R}^{d}\rightarrow\mathbb{R}$ which can
be represented as
\begin{equation}\label{eq:f_high_D}
    f\left(\xvec\right)=\int_{\mathbb{S}^{d-1}\times\mathbb{R}}d\wvec db\left(\alpha\left(\wvec,b\right)\left[\wvec^{\top}\xvec-b\right]_{+}\right)
\end{equation}
for some measure $\alpha$ we can write several useful relations between its derivatives and
the measure $\alpha$. First, we have that}
\begin{equation}
\frac{\partial}{\partial x_{i}}f\left(\xvec\right)=\int_{\mathbb{S}^{d-1}\times\mathbb{R}}\diff \wvec \diff b\left(w_{i}\alpha\left(\wvec,b\right)H\left(\wvec^{\top}\xvec-b\right)\right)
\end{equation}
where we recall $H$ is the Heaviside step function, and 
\begin{align*}
\frac{\partial}{\partial x_{i}}\frac{\partial}{\partial x_{j}}f\left(\xvec\right) & =\int_{\mathbb{S}^{d-1}\times\mathbb{R}}\diff\wvec \diff b\left(w_{j}w_{i}\alpha\left(\wvec,b\right)\delta\left(\wvec^{\top}\xvec-b\right)\right)\\
 & =\int_{\mathbb{S}^{d-1}}\diff\wvec\alpha\left(\wvec,\wvec^{\top}\xvec\right)w_{j}w_{i}\,.
\end{align*}
Therefore, the Hessian can be written as
\begin{align}
\nabla^{2}f\left(\xvec\right) & =\int_{\mathbb{S}^{d-1}}\diff \wvec\alpha\left(\wvec,\wvec^{\top}\xvec\right)\wvec\wvec^{\top}\label{eq: Hessian 2 alpha}
\end{align}
and the Laplacian as
\begin{align}
\triangle f\left(\xvec\right) & =\mathrm{Tr}\left[\nabla^{2}f\left(\xvec\right)\right]=\sum_{i=1}^{d}\frac{\partial^{2}}{\partial x_{i}^{2}}f\left(\xvec\right)\nonumber \\
 & =\int_{\mathbb{S}^{d-1}}\diff\wvec\alpha\left(\wvec,\wvec^{\top}\xvec\right)\left\Vert \wvec\right\Vert ^{2}=\int_{\mathbb{S}^{d-1}}\diff\wvec\alpha\left(\wvec,\wvec^{\top}\xvec\right)\,.\label{eq: Laplacian}
\end{align}

What remains is to minimize $\norm{\alpha}_1$ under this constraint.  As an indicative result, we can calculate $\norm{\alpha}_1$ for the special case where it is non-negative, and so $f=f_{\alpha,c}$ is convex and the Hessian is p.s.d.:
\begin{claim}\label{clm:l1_2_Laplacian}
If $\alpha\left(\wvec,b\right)\geq0$, $\forall \wvec\in \mathbb{S}^{d-1}$ and $\forall b \in \mathbb{R}$,  then
\[
\left\Vert \alpha\right\Vert _{1}=\lim_{r\rightarrow\infty}\frac{1}{r^{d-1}\mathbb{V}^{d-1}}\int_{\norm{\xvec}_2\leq r}\mathrm{Tr}\left[\nabla^{2}f\left(\xvec\right)\right]d\xvec=\lim_{r\rightarrow\infty}\frac{1}{r^{d-1}\mathbb{V}^{d-1}}\oiint_{\norm{\xvec}_2= r} \nabla f(\xvec)^{\top} \diff\hat{\boldsymbol{n}}(\xvec) 
\]
where $\mathbb{V}^{d}$ is the volume of a unit radius $d$-ball, and $\diff \hat{\boldsymbol{n}}(\xvec)$ is the outward pointing unit normal.
\end{claim}
\begin{proof}
The right equality is a direct consequence of the divergence theorem. It remains to prove the left equality. From \eqref{eq: Laplacian} we have
\begin{align*}
 & \lim_{r\rightarrow\infty}\frac{1}{r^{d-1}\mathbb{V}^{d-1}}\int_{\norm{\xvec}_2\leq r}\mathrm{Tr}\left[\nabla_{\xvec}^{2}f\left(\xvec\right)\right]\diff\xvec
 =\lim_{r\rightarrow\infty}\frac{1}{r^{d-1}\mathbb{V}^{d-1}}\int_{\norm{\xvec}_2\leq r}\diff\xvec\int_{\mathbb{S}^{d-1}}\diff\wvec\alpha\left(\wvec,\wvec^{\top}\xvec\right)\\
 & =\lim_{r\rightarrow\infty}\frac{1}{r^{d-1}\mathbb{V}^{d-1}}\int_{\mathbb{S}^{d-1}}\diff\wvec\left[\int_{\norm{\xvec}_2\leq r}\diff\xvec\alpha\left(\wvec,\wvec^{\top}\xvec\right)\right]\\
 & \overset{\left(1\right)}{=}\lim_{r\rightarrow\infty}\frac{1}{r^{d-1}\mathbb{V}^{d-1}}\int_{\mathbb{S}^{d-1}}\diff\wvec\left[\int_{-r}^{r}\diff x_{1}\alpha\left(\wvec,x_{1}\right)\int_{\mathcal{A}_{r}\left(x_{1}\right)}\prod_{i=2}^{d}\diff x_{i}\right]\\
 & \overset{\left(2\right)}{=}\lim_{r\rightarrow\infty}\frac{1}{r^{d-1}\mathbb{V}^{d-1}}\int_{\mathbb{S}^{d-1}}\diff\wvec\left[\int_{\mathbb{R}}\diff x_{1}\alpha\left(\wvec,x_{1}\right)\mathcal{I}\left[x_{1}\in\left[-r,r\right]\right]\mathbb{V}^{d-1}\left(r^{2}-x_1^{2}\right)^{\frac{d-1}{2}}\right]\\
 & \overset{\left(3\right)}{=}\int_{\mathbb{S}^{d-1}}\diff\wvec\left[\int_{\mathbb{R}}\diff x_{1}\alpha\left(\wvec,x_{1}\right)\lim_{r\rightarrow\infty}\left(\mathcal{I}\left[x_{1}\in\left[-r,r\right]\right]\left[1-\left(x_1/r\right)^{2}\right]^{\frac{d-1}{2}}\right)\right]\\
 & \overset{\left(4\right)}{=}\int_{\mathbb{S}^{d-1}}\diff\wvec\int_{\mathbb{R}}\diff b \alpha\left(\wvec,b\right)
  \overset{\left(5\right)}{=}\left\Vert \alpha\right\Vert _{1}
\end{align*}
where in $\left(1\right)$ we assume, without loss of generality,
that $\wvec^{\top}=\left(1,0,\dots,0\right)$ inside the square brackets
and define the following $d-1$ dimensional ball
\begin{equation} \label{eq: unit ball}
   \mathcal{A}_{r}\left(z\right)\triangleq\left\{ \xvec\in\mathbb{R}^{d}|\left\Vert \xvec\right\Vert _{2}\leq r,x_{1}=z\right\} \,,
\end{equation}

in $\left(2\right)$ we define $\mathcal{I}$ as an indicator function
and calculate the surface area of $\mathcal{A}_{r}\left(z\right)$,
in $\left(3\right)$ we switch the order of the limit and integration
using the bounded convergence theorem, in $\left(4\right)$ we denote $x_{1}=b$, and in $\left(5\right)$ we
use our assumption that $\alpha\left(\wvec,b\right)\geq0.$
\end{proof}

We do not expect that in general $\alpha$ would be non-negative, nor that the Hessian would be p.s.d.  Recall that in the one dimensional case, $\oR(f)$ is related to the integral of the {\em absolute value} of $f''$. Therefore, in order to generalize Claim \ref{clm:l1_2_Laplacian} to mixed sign eigenvalues, a naive conjecture would be that $\oR(f)$ is given by the integral of the {\em nuclear norm} of the Hessian, up to a linear function accounting for the boundary conditions as in the one dimensional case.  That is, the sum of the absolute values of the eigenvalues of the Hessian, as opposed to the Laplacian which is the sum of signed eigenvalues.

Unfortunately, such a conjecture cannot be accurate. The issue is that, for some functions with vanishing Hessian, any norm of the Hessian, normalized as in Claim \ref{clm:l1_2_Laplacian}, would be equal to zero. For example, 
\begin{claim}\label{claim:vanishing Hessian}
For the function
\begin{equation}
h(\xvec)=\int_{\mathbb{S}^{d-1}}d\wvec\left(\left[\mathbf{w}^{\top}\mathbf{x}+1\right]_{+}-2\left[\mathbf{w}^{\top}\mathbf{x}\right]_{+}+\left[\mathbf{w}^{\top}\mathbf{x}-1\right]_{+}\right)\,,\label{eq: h(x)}
\end{equation}
for any norm, we have
\begin{equation} \label{eq:normalized Hessian}
\lim_{r\rightarrow\infty}\frac{1}{r^{d-1}}\int_{\left\Vert \xvec\right\Vert \leq r}d\mathbf{x}\left\Vert \nabla^{2}h(\xvec)\right\Vert =0\,.
\end{equation}
\end{claim}
The proof is given in appendix \ref{sec:vanishing Hessian proof}.
Importantly, this function is written as the output of an infinite neural network with a finite nonzero norm $\left\Vert \alpha\right\Vert _{1}=4\int_{\mathbb{S}^{d-1}}d\wvec>0\,$. Moreover, we can perfectly fit any finite set of points $\{\xvec_n,y_n\}_{n=1}^N$ using a linear combination of scaled and shifted $h(\xvec)$  (each $h(\xvec)$ has a radial "bump" shape, which vanishes at infinity). Therefore, the true expression for $\bar{R}(f)$, must be different from a normalized integral of some norm of the Hessian (as in eq. \ref{eq:normalized Hessian}) --- otherwise we would get $\bar{R}(f)=0$ for such a fit, in contradiction to that $\left\Vert \alpha\right\Vert _{1}>0$ in this case. We leave it to future work to find if there is some function space cost $\bar{R}(f)$ corresponding to $\norm{\alpha}_1$  minimization.

\section{Discussion}
\label{sec-disc}

As we are realizing that (explicit or implicit) regularization plays a key role in deep learning, it is crucial to understand how simple norm control on the weights in parameter space induces rich complexity control in function space.  In a sense, for infinite size networks, the {\em only} role of the architecture is to give rise to and shape this rich complexity control.  Recent work studied this question in {\em linear} neural networks, and showed how in regularizing the weights in a convolutional network yields rich complexity control inducing sparsity in the frequency domain \citep{gunasekar2018implicit}.  Here we go beyond linear networks and study infinite ReLU networks. 

Much in the same as linear convolutional neural networks can represent any {\em linear} function, and the architecture's only role is to induce complexity control in that space, infinite width ReLU networks can represent any {\em continuous} function, and the role of the architecture, in our view, is to induce complexity control over function space.  We see that indeed, even for univariate functions, the architecture already induces a natural complexity control that is not obvious nor explicit.  Furthermore, we are particularly excited that this complexity control exactly matches the learning rules studied in recent work on interpolation learning \citep{belkin2018overfitting}, and provides a concrete connection of that work to deep learning.  We are eager to find out whether this connection carries also to the multivariate case.

We also hope that our study will reinvigorate approximation theory work on neural networks, studying approximation by networks of bounded {\em norm} rather than a bounded number of units.

Similarly, we would argue that the study of the importance of depth in neural networks should focus not on gaps in the {\em size} (number of units) required to fit a function, but whether deeper networks allow lower {\em norm} representation, and how depth changes the inductive bias induced by norm control over the weights.  \citet{gunasekar2018implicit} showed that for linear convolutional networks, the depth meaningfully changes the induced inductive bias, with depth $L$ networks corresponding to an $\ell_{2/L}$ bridge penalty, but in fully connected linear network the depth has no effect.  In Appendix \ref{sec:depthEffect} we study an architecture with infinitely many deep parallel ReLU networks, and show that also for such an architecture depth $L$ gives rise to similar $\ell_{2/L}$ sparsity-inducing bridge penalties, replacing the $\ell_1$ penalty for two-layer networks.  

It would be interesting to study if and how depth changes the induced complexity control in infinitely wide fully connected $L$-layer ReLU networks.  We can naturally extend our setup, letting $C_L(\theta)$ refer to the sum of the square of all weights in the system, and defining $\oR_L$ analogously to \eqref{eq:oR}, where the infimum is over all depth $L$ ReLU networks, with any number of units per layer (similar to $\gamma^L_{2,2}(f)$ as defined by \citet{neyshabur2015norm}).  In an anecdotal empirical example presented in Figure \ref{fig:deepernets}, depth does not appear to change the inductive bias, as with any depth network we recover a function minimizing $\oR(f)$ as calculated in Theorem \ref{thm:main}.  Another natural extension is to consider multiple output units---for two layer linear networks this corresponds to Frobenious norm control on a matrix factorization which induces nuclear norm regularization on the linear mapping from input to output \citep{fazel2001rank,srebro2004MMMF}.  How does this play out for $f:\R^d\to\R^k$ in function space with ReLU networks?

\subsection*{Acknowledgements}

We are in debt to Charlie Smart (University of Chicago) for pointing out the connection to Green's functions, and would also like to thank Jason Lee (UCS), Holden Lee (Princeton) and Arturs Backurs (TTIC) for helpful discussions. PS and NS were partially supported by NSF awards 1546500 and 1764032.

\removed{

In this work, we have analyzed the representational capacity of neural networks with infinite width and bounded norm. Unlike classical approximation results which consider capacity control through constraints on the number of units, our framework concurs with common practices in deep learning, where capacity control is performed (at least implicitly) through regularization on the norm of the weights instead of imposing size constraints on the network. 

In \thmref{thm:main}, we precisely characterize the required norm to represent univariate functions $f: \R \to \R$ with 1-hidden layer ReLU networks, which turns out to be essentially the total variation of its derivative $\int \abs{f''(x)} \diff x$. It follows that a wide range of smooth functions can be approximated arbitrarily well with finite norm, even though infinitely many units are necessary. This result provides an alternative but equivalent view of infinite width neural network training: instead of parameter optimization with norm constraints, we can see it as search in function space with variational constraints. In \thmref{thm:spline}, we showed that for $1$-dimensional data, the linear interpolation over the points, when implemented by a network, achieves minimum norm while perfectly fitting the data.

In Section \ref{sec:greens}, we have presented a concrete connection between Green's functions and training with (sufficiently large) neural networks.
}

\bibliography{bibliography}

\newpage

\appendix
\part*{Appendix}
	
\section{Equivalence of overall $\ell_2$ control to $\ell_1$ control on the output layer}
\label{app:neyshaburreproof}

In this Appendix we formally show that regularizing the overall $\ell_2$ norm on the weights in two layers (i.e.~$C(\theta)$, the sum of squares of all weights in the system), is equivalent to constraining the $\ell_2$ norm of incoming weights for each unit in the hidden layer, and regularizing the $\ell_1$ norm of weights in the output layer.  This was already observed and proved for networks without an unregularized bias as Theorem 1 of \citet{neyshabur2014search}, and in somewhat more general form as Theorem 10 of \citet{neyshabur2015norm}.  The exact same arguments are valid also when an unregularized bias is allowed, and in this Appendix we make this precise, and repeat the arguments of \citet{neyshabur2014search,neyshabur2015norm} for completeness.

Recall the definition of 2-layer ReLU networks $\func{\params}$ for $\params \in \ThetaClass_2$:
\begin{equation*}
	\func{\params}(\xvec) = \sum_{i=1}^k w^{(2)}_i \relu{\ip{\wvec^{(1)}_i}{\xvec} + b^{(1)}_i} 
	+ 
	b^{(2)}
\end{equation*}

\begin{lemma}

\eqref{eq:Rf} and \eqref{eq:finitel1} are equivalent. More specifically:

\begin{equation*}
\begin{aligned}
& {\inf_{\params \in \ThetaClass_2}}
& & ~\half \sum_{i=1}^k \left( (w^{(2)}_i)^2 + \norm{\wvec^{(1)}_i} _2^2 \right) \quad = 
& & {\inf_{\params \in \ThetaClass_2}} 
& & ~\norm{\wvec^{(2)}}_1 
\\
& ~\text{\st} 
& & ~\func{\params} = f 
& & ~\text{\st}
& & ~h_\params = f~,~ \forall{i}: \norm{\wvec^{(1)}_i}_2 = 1
\end{aligned}
\end{equation*}
\label{lemma-pal}
\end{lemma}

\begin{proof}
    For any $\params \in \ThetaClass_2$, consider the rescaled parameters $\tilde \params$ given by $\tilde \wvec^{(1)}_i = c_i \wvec^{(1)}_i $, $\tilde w^{(2)}_i = \frac{w^{(2)}_i}{c_i}$, $\tilde b^{(1)}_i = c_i b^{(1)}_i$, for some $c_i > 0$. Now, check that, for all $i$:
    \begin{equation*}
    \begin{split}
	    \tilde w^{(2)}_i \relu{\ip{\tilde \wvec^{(1)}_i}{\xvec} + \tilde b^{(1)}_i}
	    =
	   \frac{w^{(2)}_i}{c_i} \relu{c_i \left(\ip{ \wvec^{(1)}_i}{\xvec} + b^{(1)}_i \right)}
	    =w^{(2)}_i \relu{\ip{ \wvec^{(1)}_i}{\xvec} + b^{(1)}_i}
	\end{split}
    \end{equation*}
	Therefore $\func{\params} = \func{\tilde \params}$. Moreover, we have that, from the inequality between arithmetic and geometric means:
    \begin{equation*}
    \begin{split}
	\half \sum_{i=1}^k
    \left( (w^{(2)}_i)^2 + \norm{\wvec^{(1)}_i} _2^2\right) 
	\geq 
	\sum_{i=1}^k
    \abs{w^{(2)}_i} \cdot \norm{\wvec^{(1)}_i} _2
	\end{split}
    \end{equation*}
	where a rescaling given by $c_i = \sqrt{\abs{w^{(2)}_i} / \norm{\wvec^{(1)}_i}_2 }$ minimizes the left-hand side and achieves equality. Since the right-hand side is invariant to rescaling, we can arbitrarily set $\norm{\wvec^{(1)}_i} _2 = 1$ for all $i$, yielding $\sum_{i=1}^k \abs{w^{(2)}_i} = \norm{\wvec^{(2)}}_1 $.
\end{proof}
\section{Relationship to Barron's Analysis}
\label{app:barron}

\citet{barron1993universal,barron1994approximation} studies function approximation using two-layer neural network with sigmoidal activation, bounding both the {\em number} of units required for approximation, and also the $\ell_1$ {\em norm} of the weights {\em in the output layer}.  For a function $f:\R^d\rightarrow\R$, Barron defined a quantity $C_f$, which we refer to as the ``Barron Norm''\footnote{The ``Barron Norm'' can be finite or infinite, and is a semi-norm over the convex set of functions over which it is finite.}.  The core component of Barron's analysis is showing how to approximate $f$ with an infinite number of units, using a measure $\alpha$ over weights, similar to our representation $h_{\alpha,c}$ as in \eqref{eq:alpha} (but with different activation functions), and bounding the $\ell_1$ norm $\norm{\alpha}_1$ of this measure\removed{\foonote{Barron referred to this as the scaling, since he used probability distributions instead of signed measures.}} in terms of $C_f$.  Approximations using a finite number of units can then be obtained by sampling from this measure.  Barron's analysis is therefore in many ways similar to ours, suggesting the Barron Norm $C_f$ as the induced complexity measure in function space.  However, we point out several important differences between Barron's approach and ours.

One important difference is that while Barron's norm $C_f$ controls the norm of the output layer, it does {\em not} control the overall norm of the weights across both layers.  To understand this better, recall that Barron's analysis is based on first considering cosine activations, then approximating the cosine activations with step functions, and finally approximating the step functions with sigmoidal units.  Working with cosine activations, Barron {\em does} control a norm of the ``weights'' of each cosine, and so does in a sense provide overall norm control.  However, when approximating the cosine with step functions and then in turn with sigmoids, the norm of the weights of the sigmoidal units must increase to infinity in order for them to approximate step functions.  The resulting sigmoidal network, although having controlled norm in the output layer, has weights going to infinity in the hidden layer, and thus the overall norm of the network increases to infinity in an controlled and unanalyzed way as we seek better and better approximations (i.e.~as $\norm{f-h_{\alpha,c}}\rightarrow 0$).  Controlling the norm of the output layer without controlling the norm in the hidden layer is meaningful for sigmoidal networks because the output of each unit is bounded regardless of its weights, and so the $\ell_\infty$ norm of the output vector of the hidden layer is always bounded, hence bounding the $\ell_1$ norm of the output layer's weights is meaningful.  For ReLU networks, it does not make sense to control the weights in one layer without the other, since the activation is positive-homogeneous and the scale of the weights interact.  \removed{This homogeneity is also key in establishing the equivalence between overall $\ell_2$ norm control on both layers, and $\ell_1$ control over the top layer (as in Lemma \ref{lemma-pal} following \citealt{neyshabur2014search}), and so this equivalence only holds for networks with homogeneous activations, for which Barron's analysis would fail.}

Furthermore, even though Barron's norm $C_f$ does provide an upper bound on the $\ell_1$ norm of the top layer (if we ignore the norm of the weights in the hidden layer), for step or sigmoidal activation this upper bound is not tight.  And so, even if we were to regularized only the norm of the the output layer in a sigmoidal network, the induced complexity control in function space is {\em not} captured by the Barron norm $C_f$ ($C_f$ is only an upper bound on the induced complexity function, and so its form or behaviour might be radically different).  Viewed as an approximation theory result, it provides a sufficient, but not a necessary condition for approximability, and so is not, for example, sufficient for studying depth separation.  E.g., \cite{lee2017ability} provide depth separation results for a generalization of Barron's norm, but this does {\em not} translate to any meaningful depth separation results for sigmoidal (and certainly not ReLU) neural networks.

A final technical but important issue is that Barron's norm is only useful for approximation results when applied to functions over the entire space $\R^d$, but provides approximation guarantees only in a ball of bounded radius, with a strong dependence on this radius.  That is, to study approximability of a function inside a bounded ball, the true quantity this suggests is the minimum Barron norm over all extensions of the function to the entire space, but it is not clear how such a minimum norm extension would behave.  Furthermore, the Barron norm is specific to approximation over Euclidean balls.  It can be generalized also to approximation over other compact domains, but the shape of the domain would then change the definition of the norm.

\section{The effect of neural networks depth on the inductive bias}
\label{sec:depthEffect}

\natinote{First say we can also define the $C_L$ and therefore $\oR_L$ also for depth $L$.  But for deeper networks its no longer convex \cite{neyshabur2015norm} and we do not have a representation in terms of measures.  Does it change the bias? 
Not clear, e.g. for fully connected linear it does not, but for linear convolutional or with multiple outputs it does.  Some anecdotal experiments suggest for 1d perhaps the bias is similar (refer to figures, no need to discuss them elsewhere except in their captions). 
To give a sense of the possible effect of depth, we study here a specific deep and infinitely wide architecture, which is simpler then fully connected, and has seen success in practice \todo{\cite{?}} .    }
\dnote{Nati - is anything from the previous comment relevant? If not, just delete.}
We now turn to a different infinite width architecture,
where we study the effect of depth.
The networks we consider have a parallel structure
common in many state of the art
system \citep{Cheng2016,Shazeer2017}.
Consider a deep neural network with $L$ layers and a parallel architecture, so the output is a sum of $\subnetNum$ sub-networks with $L-1$ layers.
Such a network is parameterized by
$\params=
\left(
    \subnetNum,
    \weights_1, \dots, \weights_{\subnetNum},
    \coef
\right)$,
where 
$\weights_i=\left(\lweights{1}{i},\dots,\lweights{L-1}{i}\right)$%
is the set of weight matrices of the $i$\Th network,
and $\coef$ is the weight vector
of the last layer,
linearly combining
the (scalar) outputs of all $\subnetNum$ parallel sub-networks.
The parameter class $\ThetaClass_L$ is therefore defined as
\begin{multline*}
    \label{eq:ThetaClass2}
     \ThetaClass_L = 
    \left\{ 
        \params=
        \left(
            \subnetNum,
            \weights_1, \dots, \weights_{\subnetNum},
            \coef
        \right)
        ~ \right| ~ 
        k\in\N,
        \coef\in\R^k, 
        \\
        \left.
        \forall{i\in\iter{k}
        }:
        \prn{
            \begin{aligned}
                &\lweights{1}{i}\!\in\R^{\matSize \times d},
                \\
                &\lweights{2}{i}\!\in\R^{\matSize \times \matSize},
                \ldots,
                \lweights{L-2}{i}\!\in\R^{\matSize \times \matSize},
                \\
                &\lweights{L-1}{i}\!\in\R^{1\times\matSize}
            \end{aligned}
        }
    \right\},
\end{multline*}
    where $\matSize\in\N$ is the (fixed) width of the layers of the parallel networks.

The network function is defined as
\begin{equation}
    \func{\params}^L\prn{\xvec}
    =
    \sum_{i=1}^{\subnetNum}
    {
    \coef[i] \subfunc{\weights_i}{\xvec}
    }~,
\end{equation}
where
\begin{equation}
	\subfunc{\weights_i}{\xvec}
	= 
	\left(
    	h^{(L-1)}_{\lweights{L-1}{i}} 
    	\circ \dots \circ
    	h^{(1)}_{\lweights{1}{i}}
	\right)
	\prn{\xvec}~,
\end{equation}
and for $i\in\iter{\subnetNum}, l\in\iter{L-1}$:
\begin{equation}
	h^{(i)}_{\lweights{l}{i}}(\xvec)= \relu{\lweights{l}{i} \xvec}~,
\end{equation}

Unlike the networks in \secref{sec:reluNetworks},
the networks in this section do \textbf{not} have biases.

The squared Euclidean norm of the weights  (averaged over all layers) is now defined as
\begin{align}
    \Eucost_L(\params) 
    &= 
    \frac{1}{L}
    \left(
        \norm{\coef}_2^2
        +
        \sum_{i=1}^{\subnetNum}
        \sum_{\lyrIdx=1}^{L-1}
        \norm{\lweights{\lyrIdx}{i}}_F^2
    \right)~,
\end{align}
yielding the following definition of the infimum norm required to implement a function with this parallel architecture:
%
%
\begin{align}
\label{eq:parllelImplmntCost}
\begin{split}
\RP_L(f)
=
{\inf_{\params\in\ThetaClass_L}~}
\Eucost_L(\params) 
 ~~\text{\st}~
\func{\params}^L
= 
f~.
\end{split}
\end{align}
\removed{
\inote{Similarly to before we define...}
\begin{align}
\begin{split}
\oR_L^{\subnetNum}\prn{f} = \lim_{\epsilon \to 0} \left( 
{\inf_{\params \in \ThetaClass}}~
\Eucost_L(\params) 
 ~\text{\st}~
\norm{h^{\subnetNum}_{\params} - f}_\infty \leq \epsilon \right)
\end{split}
\end{align}
}

In the following Theorem we show that the $\ell_2$ norm control can be written equivalently in terms of minimizing a sparsity-inducing $\ell_p$ bridge penalty $\norm{\alpha}_p=\left(\sum_i \alpha_i^p\right)^{1/p}$ where $p=2/L<1$ (so we slightly abuse the norm notation as this is non-convex and thus not a norm) in the last layer, while the weights of the subnetworks are restricted to $\sphere$, the direct product of $L-1$ Euclidean unit spheres,
    each corresponding to the size of the matrix at its layer:
    \begin{equation}
    \label{eq:sphereDefinition}
    \sphere = 
        \mathbb{S}^{d \cdot \matSize}
        \times
        \mathbb{S}^{\matSize^2}
        \times\ldots\times
        \mathbb{S}^{\matSize^2}
        \times
        \mathbb{S}^{\matSize}~.
    \end{equation}
    In other words,
    for each 
    $\unit=\left(\lweights{1}{},
    \ldots,
    \lweights{L-1}{}\right)\in\sphere$,
    the weight matrices are normalized,
    \ie $\forall{l}: \norm{\lweights{l}{}}_F=1$. It will be convenient to let $\valpha$ denote the weights of the last layer (originally $\coef$) in this setting.
\begin{thm}
\label{prop:complexityRatio}
If \eqref{eq:parllelImplmntCost} is attainable, then
    \begin{align}
    \label{eq:lastLayerNormMin}
        \RP_L(f)
        &=
        {\inf_{\params=(k,\left\{\unit_i\right\}_{i=1}^{\subnetNum}, \valpha)\in\Theta_L}}~
        {\norm{\valpha}_{\order}^{\order}} 
         ~~\text{\st}~
        \func{\params}^L
        = 
        f, 
         ~\forall{i\in\iter{\subnetNum}}: \unit_i \in \sphere
    \end{align}
     where $\sphere$ is defined in \eqref{eq:sphereDefinition}, and for each solution of one problem (either \eqref{eq:lastLayerNormMin} or \eqref{eq:parllelImplmntCost}) we can find an equivalent solution of the other (with the same $\func{\params}^L$).
\end{thm}

\begin{proof}
The proof appears in \appref{app:implementationCostProof}.
\end{proof}
The proof of this Theorem has the immediate implication that
\begin{cor}
[Parameter alignment]
Consider any parallel architecture $\func{\params}^L$,
with $\subnetNum$ parallel $L$-layer networks, 
which minimizes
\eqref{eq:parllelImplmntCost},
\ie $\func{\params}^L=f$
and
$\Eucost_L(\params)=\RP_L(f)$.
Then in each parallel sub-network,
equality holds between
the $\Lp{2}$-norms 
of all its (inner) layers
and the its corresponding coefficient 
in the last layer.
That is, $\forall{i\in\iter{\subnetNum}}$:
\begin{align}
    \norm{\optlweights{1}{i}}_F^2
    =
    \ldots
    =
    \norm{\optlweights{L-1}{i}}_F^2
    =
    \abs
    {
        \optcoef[i]
    }^{2}
\end{align}
\end{cor}
\begin{proof}
During the proof of 
\thmref{prop:complexityRatio}
we used the inequality of arithmetic and geometric means
in \eqref{eq:meansIneq}
to show that $\forall{i\in\iter{\subnetNum}}:$
\begin{align*}
\sqrt[L]{
    \abs
    {
        \optcoef[i]
    }^{2}
    \prod_{l=1}^{L-1}
    {
        \norm{\optlweights{l}{i}}_F^2
    }
}
\le
\frac{1}{L}
\left(
    \abs
    {
        \optcoef[i]
    }^{2}
    +
    \sum_{l=1}^{L-1}
    {
        \norm{\optlweights{l}{i}}_F^2
    }
\right)
\end{align*}
After finishing that proof,
we know that it must hold with equality.
But, equality holds if and only if
all the numbers in the means are equal.
\end{proof}


\removed{
We now wish to generalize \eqref{eq:lastLayerNormMin} to a measure
that is possibly non-countable infinite.
Like \citet{rosset2007l1}, 
we start by adding a positivity constraint by using the doubling trick,
and claim equivalence between the optimization problems.

\natinote{Skip this Lemma.  We didn't do it in Section 2 either.  I'd keep this {\bf very} short, state $\oR$ as an approximation (which you didn't) and then give the measure version.  Though I have to admit I am not entirely confident about the switch to measures here.  Perhaps its easier here to just stick with unbounded finite?  I guess we don't get anything with the measure, its just showing we can do fancy tricks with advanced math for no good reason.  Maybe best to just remove the entire measure treatment and give the sparsity for unbounded.}

\begin{lemma}[\temp{give name}]
\label{lem:positivityConstraint}
The following positive constrained problem on the \temp{signed set} $\msrDomain\times\left\{-1,1\right\}$, is equivalent to the same problem \eqref{eq:lastLayerNormMin} without the positivity constraint,
and also hold:
\begin{align}
\label{eq:withPosConstraints}
\begin{split}
\oR_L(f)
=
{\temp{\inf_{k, \tilde{\valpha}, \left\{\unit_i\right\}_{i=1}^{\subnetNum}}}~}
{\norm{\tilde{\valpha}}_{\order}^{\order}} 
 ~~\text{\st}~
\tilde{h}_{\tilde{\valpha}, \left\{\unit_i\right\}_{i=1}^{\subnetNum}}
= 
f, 
 ~\forall{i\in\iter{\subnetNum}}: \unit_i \in \sphere,
 ~\temp{\tilde{\valpha} \succeq 0}
\end{split}
\end{align}
where $\tilde{\valpha}: \temp{\sphere\times\left\{-1,1\right\}}\to\R$,
and $\tilde{\subfunc{\weights}{\xvec}}=
\sign\prn{\valpha\prn{\weights}}\subfunc{\weights}{\xvec}$
\inote{The notations here are temporary, will be improved}.
\end{lemma}

\begin{proof}
The proof appears in \appref{app:implementationCostProof}.
\end{proof}

We can thus limit ourselves to positive coefficients only, 
and generalize $\valpha$ from a coefficient vector in $\R^{\subnetNum}$
to a positive measure on $\msrDomain$.
\inote{
Copy pasted from Rosset:
Let $\left(\Omega, \Sigma\right)$ be a measurable space,
where we require $\Sigma\supset\left\{\left\{\omega\right\}: \omega \in \Omega\right\}$,
\ie the sigma algebra $\Sigma$ contains all singletons (this is a very mild assumption,
which holds for example for the "standard" Borel sigma algebra).
Let ...
}

\begin{align}
\label{eq:inftyGeneralized}
\begin{split}
\hat{\oR_L(f)}
=
{\temp{\inf_{k, \alpha, \left\{\unit_i\right\}_{i=1}^{\subnetNum}}}~}
{\norm{\alpha}_{\order}^{\order}} 
 ~~\text{\st}~
\intop_{\Omega}
\subfunc{\unit}{\xvec}
\diff
\alpha\prn{\unit}
= 
f~,
\end{split}
\end{align}
where $\norm{\valpha}_{\order}^{\order}
=
\intop_{\Omega}{\abs{\alpha}^{\order}\diff{?}}$.

}

\removed
{
\newpage
\subsubsection{Previous informal definition}
\temp{Consider now the bridge penalty over parallel architectures with
infinite parallel networks...
\todo{discuss convex NN? add cites}}
%
%
\begin{align}
\begin{split}
\lifcost{L}{\infty}{f}=
{\temp{\inf_{\valpha}}~}
{\norm{\valpha}_{\order}^{\order}} 
 ~\text{\st}~
h_{\ouralpha{}{}}^{\infty} = f 
\end{split}
\end{align}
where we extend the notation of parallel architectures
to the infinite-width case
\begin{equation}
    \func{\valpha}[\xvec]{\infty}
    =
    \intop_{\sphere}
    \valpha\left(\unit\right)
    \subfunc{\unit}{\xvec}~,
\end{equation}
make $\ouralpha{}{}$ infinite as well,
\ie $\ouralpha{}{i}\in\lspace$,
define the norm
\begin{align}
    \norm{\ouralpha{}{}}_{\order}^{\order}
    =
    \intop_{\sphere}
    \abs{\ouralpha{}{}\prn{\unit}}^{\order},
\end{align}
and finally define $\sphere$ as in \eqref{eq:sphereDefinition}.
}

Next, we consider minimizing an $\Lp{2}$-regularized loss over a finite set of $\samples$ samples,
\ie
\begin{align}
\label{eq:normalizedLossMinimization}
\begin{split}
\inf_{\theta\in\Theta_L}~
& {
    \sum_{\smplIdx=1}^{\samples} 
    \loss\prn{\func{\params}^L\prn{\xvec_{\smplIdx}}, y_{\smplIdx}}
    +
    \lambda \cdot \Eucost_L(\params)
} 
~,
\end{split}
\end{align}
where $\loss$ is a differentiable convex instantaneous loss function.
From \thmref{prop:complexityRatio}, it follows that \eqref{eq:normalizedLossMinimization} is equivalent to minimizing an $\Lp{{\order}}$-regularized
loss over the same set \footnote{Take any minimizer $\params^*$ of \eqref{eq:normalizedLossMinimization} and check that it must attain minimum $\Eucost_L(\params)$ when constrained to perfectly fit $\left(\xvec_\smplIdx, h_{\params^*}(\xvec_\smplIdx)\right)_{\smplIdx = 1}^\samples$. The claim then follows from \thmref{prop:complexityRatio} and the definition of $\RP_L$.},
\ie
\begin{align}
\label{eq:finiteConvexLossMinimization}
\begin{split}
{\inf_{\params=\prn{
    k, \left\{\unit_i\right\}_{i=1}^{\subnetNum},\valpha}\in\Theta_L
}}~
{
    \sum_{\smplIdx=1}^{\samples} 
    \loss\prn{
        \func{\params}^L
        \prn{\xvec_{\smplIdx}}, y_{\smplIdx}}
    +
    \lambda \cdot 
    \norm{\valpha}_{\order}^{\order}
}
~~\st~
\forall{i\in\iter{\subnetNum}}: \unit_i \in \sphere
 ~. 
 \end{split}
\end{align}
\removed{
\begin{align}
\label{eq:convexLossMinimization}
\begin{split}
{\inf_{\valpha}~}
& {
    \sum_{\smplIdx=1}^{\samples} 
    \loss\prn{
        \func{\params}^L\prn{\xvec_{\smplIdx}}\inote{\infty}, y_{\smplIdx}}
    +
    \lambda 
    \norm{\valpha}_{\order}^{\order}
}
 ~. 
 \end{split}
\end{align}

We now establish that
the regularized loss minimization problem
over infinite parallel networks
\eqref{eq:convexLossMinimization}
is equivalent to a similar problem 
on finite architectures.
}

\removed{
Our next theorem links learning on finite sample sets, when performing regularization on all parameters, and on the last layer only.
\begin{thm}
[Equivalence of learning on finite sample sets]
\label{thm:equivalenceOnFiniteSets}
Both of the above problems \eqref{eq:normalizedLossMinimization}
and 
\eqref{eq:finiteConvexLossMinimization}
are equivalent.
\removed{is equivalent to the following 
regularized loss minimization problem
over a finite number of parallel networks $\subnetNum$,
as long as $\subnetNum \ge \samples$:
\begin{align}
\label{eq:finiteConvexLossMinimization}
\begin{split}
{\inf_{\params=\prn{
    k, \left\{\unit_i\right\}_{i=1}^{\subnetNum},\valpha}
}}~
{
    \sum_{\smplIdx=1}^{\samples} 
    \loss\prn{
        \func{\params}^L
        \prn{\xvec_{\smplIdx}}, y_{\smplIdx}}
    +
    \lambda 
    \norm{\valpha}_{\order}^{\order}
}
~~\st~
\forall{i\in\iter{\subnetNum}}: \unit_i \in \sphere
 ~. 
 \end{split}
\end{align}
}
\end{thm}

\begin{proof}
The proof appears in \appref{app:implementationCostProof}.
\end{proof}
}

\removed{
\natinote{I don't understand this corollary.  Aren't the two equivalent for any N.  This has nothing to do with sparsity.}
\inote{The original idea was that 
\eqref{eq:convexLossMinimization}
is defined on \textbf{infinite} networks 
and 
\eqref{eq:finiteConvexLossMinimization}
is equivalent but on finite networks. Now that we optimize here over $k$ as well,
there's no need for them to be separate.}
}

When learning on finite sample sets,
the optimal solutions can be shown to be 
supported on bounded vector sets.
More specifically,
\begin{thm}
\label{thm:lastLayerSupportBound}
For any lower semi-continuous loss $\ell$, if ~\eqref{eq:finiteConvexLossMinimization} is attainable, then it has an optimal solution 
where $\norm{\valpha}_0 \le \samples$  \removed{$\subnetNum^* \le \samples$}.
Moreover, when $L\ge 3$, \textbf{all} optimal solutions of \eqref{eq:finiteConvexLossMinimization} have $\norm{\valpha}_0 \le \samples$.
\end{thm}

\begin{proof}
The proof appears in \appref{app:implementationCostProof}.
\end{proof}

This helps us shed light on the behavior of loss minimization when regularizing all parameters, and not only the ones on the last layer.

\removed{
\begin{cor}
When $L\ge3$,
in all the optimal solutions of
\eqref{eq:normalizedLossMinimization}
we have that $\subnetNum^* \le \samples$.
\end{cor}

\begin{proof}
Since we show 
in \thmref{prop:complexityRatio}
how to build an optimal solution of \eqref{eq:finiteConvexLossMinimization}
from an optimal solution of \eqref{eq:normalizedLossMinimization}
(with an identical number of parallel networks),
and following 
\thmref{thm:lastLayerSupportBound}.
\end{proof}
}

\removed{
The next corollary connects the above two regularized loss
minimization problems.
\begin{cor}
\label{cor:lossEquivalence}
When the width of the finite architecture is 
at least the number of samples,
\ie $m\ge \samples$,
then there is an equivalence between 
minimizing \eqref{eq:normalizedLossMinimization}
and
\eqref{eq:convexLossMinimization}.
\end{cor}

\begin{proof}
The proof appears in \appref{app:implementationCostProof}.
\end{proof}
}

As the networks become deeper,
the regularization term 
in \eqref{eq:finiteConvexLossMinimization},
$\norm{\ouralpha{}{}}_{\order}^{\order}$,
converges to an $\Lp{0}$-regularizer on $\ouralpha{}{}$.
That is, this regularizer essentially
induces sparsity in the weights of the last linear layer.
Moreover, when $L$ is indeed large enough
and the regularizer practically
behaves like an $\Lp{0}$-regularizer,
all optimal solutions to 
\eqref{eq:finiteConvexLossMinimization}
should have the same number of non-zero
weights in the last layer.
This implies that all optimal solutions of
\eqref{eq:normalizedLossMinimization}
also have the same number of non-zero 
weights in their last layer
(otherwise our construction 
\eqref{eq:constructionOfCoef}
in the proof of \thmref{prop:complexityRatio}
will yield a suboptimal solution which is impossible).
This means the $\Lp{2}$-regularized loss minimization problem
\eqref{eq:normalizedLossMinimization}
will implicitly 
zero out as many sub-networks
in the parallel architecture as possible. This result is closely related to the result of 
\cite{gunasekar2018implicit}, which found a similar inductive bias in certain \textit{linear} convolutional neural nets. Here we show such a sparsity-inducing bias (which gets stronger with depth) also affects non-linear deep networks.


\subsection*{Proofs}
\label{app:implementationCostProof}

\paragraph{Proof for \thmref{prop:complexityRatio}}
$ $ \newline
\begin{proof}
In the following proof
we show that given an optimal solution of either
of the two problems,
one can construct a feasible solution
to the second one,
with an equal objective value.
We follow a proof by \citet{wei2018margin}, 
who used a similar construction to bind the max margin 
of $\Lp{2}$-regularized neural networks with one hidden layer,
with the max margin of convex neural networks 
with one (infinite) hidden layer
and an $\Lp{1}$-regularization over the last layer.

\begin{description}
  \item[$\bullet~~{
        \Eucost_L\prn{\params^*} \ge \norm{\valpha^*}^{\order}_{\order}
  }$:]
  { 
    Given a solution 
    $$\params^{*}=
    \left(
        \subnetNum^*,
        \weights_1^{*}, \dots, \weights_{\subnetNum^*}^{*},
        \optcoef
    \right)$$
    to \eqref{eq:parllelImplmntCost},
    we show how to construct a solution for \eqref{eq:lastLayerNormMin}.
    Start by setting the following solution,
    where $\forall{i}\in\iter{\subnetNum^*}$:
    \begin{align}
    \begin{split}
    \label{eq:alphaConstruction}
        \unit_i
        &=\left(
        \optulweights{1}{i},
        \dots,
        {\optulweights{L-1}{i}}
        \right)
        \\
        \ouralpha{}{i}
        &=
        \optcoef[i]\cdot
        \norm{\optlweights{1}{i}}_F
        \cdots
        \norm{\optlweights{L-1}{i}}_F~.
    \end{split}
    \end{align}
    %
    %
    We use the $1$-positive-homogeneity 
    of the activation functions to show 
    $\forall{i\in\iter{\subnetNum^*}}$
    it holds that
    \begin{align}
    \begin{split}
       \subfunc{ \unit_i^{*}}{\xvec}
        \prod_{l=1}^{L-1}
            {
                \norm{\optlweights{l}{i}}_F
            }    
       &=
            \subfunc{
                \left(
                    \optulweights{1}{i},
                    \dots,
                    \optulweights{L-1}{i}
                \right)}{\xvec}
            \cdot
            \prod_{l=1}^{L-1}
            {
                \norm{\optlweights{l}{i}}_F
            }
        \\
        &=
        \subfunc{
            \left(
            \norm{\optlweights{1}{i}}_F
            \optulweights{1}{i},
            \dots,
            \norm{\optlweights{L-1}{i}}_F
            \optulweights{L-1}{i}
            \right)\!
        }{\xvec}
        \\
        &=
            \subfunc{ \weights_i^{*}}{\xvec}~.
    \end{split}
    \end{align}
    Now note that the constructed solution is feasible,
    since  
    $\unit_i \in \sphere, \forall{i\in\iter{\subnetNum^*}}$ and
    $\forall{\xvec\in\featureSpace}$ it holds
    \begin{align}
    \begin{split}
        \func{\params=\prn{k^*, \left\{\unit_i\right\}_{i=1}^{\subnetNum^*,\valpha}}}^L
        \prn{\xvec}
        &=
        \sum_{i=1}^{\subnetNum^*}
        \ouralpha{}{i}
        \subfunc{\unit_i}{\xvec}
        =
        \sum_{i=1}^{\subnetNum^*}
        {
            \optcoef[i]\cdot
            \prod_{l=1}^{L-1}
            {
                \norm{\optlweights{l}{i}}_F
            }
            \cdot
            \subfunc{ \unit_i^{*}}
                {\xvec}
        }
        \\
        &=
        \sum_{i=1}^{\subnetNum^*}
        {
            \optcoef[i]
            \cdot
            \subfunc{ \weights_i^{*}}{\xvec}
        }
        =
        \func{\params^*}^L\prn{\xvec}
        =
        f\prn{\xvec}
        ~.
    \end{split}
    \end{align}
    We now show that the value of the new solution,
    \ie its norm, equals
    \begin{align}
    \begin{split}
    \label{eq:meansIneq}
        \norm{\valpha}_{\order}^{\order}
        &=
        \sum_{i=1}^{\subnetNum^*}
        {
            \abs
            {
                \optcoef[i]
                \cdot
                \prod_{l=1}^{L-1}
                {
                    \norm{\optlweights{l}{i}}_F
                }
            }^{\order}
        }
        =
        \sum_{i=1}^{\subnetNum^*}
        {
            \left(
                \abs
                {
                    \optcoef[i]
                }^{2}
                \cdot
                \prod_{l=1}^{L-1}
                {
                    \norm{\optlweights{l}{i}}_F^2
                }
            \right)^{\frac{1}{L}}
        }
        \\
        &\le
        \sum_{i=1}^{\subnetNum^*}
        {
            \frac{1}{L}
            \left(
                \abs
                {
                    \optcoef[i]
                }^{2}
                +
                \sum_{l=1}^{L-1}
                {
                    \norm{\optlweights{l}{i}}_F^2
                }
            \right)
        }
        \\
        &
        =
        \frac{1}{L}
        \left(
            \norm{\optcoef}_2^2
            +
            \sum_{i=1}^{\subnetNum^*}
            \sum_{\lyrIdx=1}^{L-1}
            \norm{\optlweights{\lyrIdx}{i}}_F^2
        \right)
        =
        \Eucost_L\prn{\params^*}~,
    \end{split}
    \end{align}
    where
    we used the inequality of arithmetic and geometric means.

    As a conclusion, 
    we get the required inequality
    \begin{align}
    \begin{split}
        \norm{\valpha^*}_{\order}^{\order} 
        \le
        \norm{\valpha}_{\order}^{\order} 
        \le
        \Eucost_L\prn{\params^*}~.
    \end{split}
    \end{align}
  }
  %
  %
  \item[$\bullet~~{
        \Eucost_L\prn{\params^*} \le \norm{\valpha^*}^{\order}_{\order}
  }$:]
  {
    When \eqref{eq:lastLayerNormMin}
    is attainable,
    and
    given an optimal solution
    $\params=(\subnetNum^*, \valpha^*, \left\{\unit_i^*\right\}_{i=1}^{\subnetNum^*})$,
    we show how to attain a solution to
    the implementation cost formulation
    \eqref{eq:parllelImplmntCost}.
    
    We construct a solution $\params$ where $\forall{i\in\iter{\subnetNum^*}}$
    \begin{align}
    \label{eq:constructionOfCoef}
    \coef[i]=&
    \sign\prn{\ouralpha{*}{i}}
    {
        \abs{\ouralpha{*}{i}}^{\nicefrac{1}{L}}
    }
    \\
    \weights_i=&
    {
        \abs{\ouralpha{*}{i}}^{\nicefrac{1}{L}}
    }
    \cdot
    \unit_i^*~.
    \end{align}
    
    By using the $1$-positive-homogeneity
    property of the activation functions
    (once for every matrix in~$\unit_i$),
    we get that for all ${i\in\iter{\subnetNum^*}}$:
    \begin{align}
    \begin{split}
        \subfunc{\weights_i}{\xvec}
        =
        \subfunc{
        \abs{\ouralpha{*}{i}}^{\frac{1}{L}}  
        \unit_i^*}{\xvec}
        = 
        \abs{\ouralpha{*}{i}}^{\frac{L-1}{L}}
        \subfunc{\unit_i^*}{\xvec}
    \end{split}
    \end{align}
    Now, 
    the constructed solution can be shown to be feasible,
    since $\forall{\xvec\in\featureSpace}$:
    \begin{align}
    \begin{split}
        \func{\params}^L\prn{\xvec}
        &=
        \sum_{i=1}^{\subnetNum^*}
        {
        \coef[i] \subfunc{\weights_i}{\xvec}
        }
        =
        \sum_{i=1}^{\subnetNum^*}
        {
            \sign\prn{\ouralpha{*}{i}}
             \abs{\ouralpha{*}{i}}^{
                \frac{1}{L}+\frac{L-1}{L}}
            \subfunc{\unit_i^*}{\xvec}
        }
        \\
        &=
        \sum_{i=1}^{\subnetNum^*}
        {
            \sign\prn{\ouralpha{*}{i}}
            \abs{\ouralpha{*}{i}}
            \subfunc{\unit_i^*}{\xvec}
        }
        %
        =
        \sum_{i=1}^{\subnetNum^*}
        {
            \ouralpha{*}{i}
            \subfunc{\unit_i^*}{\xvec}
        }
        \\
        &=
        \func{\params=\prn{
            \subnetNum^*,            \left\{\unit_i^*\right\}_{i=1}^{\subnetNum^*,\valpha^*}}}^L\prn{\xvec}
        =
        f\prn{\xvec}~.
    \end{split}
    \end{align}
    The value of this solution is
    \begin{align}
    \begin{split}
    \Eucost_L\prn{\params}
    &=
    \frac{1}{L}
    \left(
        \norm{\coef}_2^2
        +
        \sum_{i=1}^{\subnetNum^*}
        \sum_{\lyrIdx=1}^{L-1}
        \norm{\lweights{\lyrIdx}{i}}_F^2
    \right)
    \\
    &
    =
    \frac{1}{L}
    \sum_{i=1}^{\subnetNum^*}
    \left(
        \abs{
        \ouralpha{*}{i}
        }^{\order}
        +
        \abs{
        \ouralpha{*}{i}
        }^{\order}
        \cdot
        \sum_{\lyrIdx=1}^{L-1}
        \underbrace{
            \norm{\ulweights{\lyrIdx}{i}}_F^2
        }_{=1}
    \right)
    \\
    &=
    \frac{1}{L}
    \sum_{i=1}^{\subnetNum^*}
    \left[{
        \abs{\ouralpha{*}{i}}^{\order}
        \left(
            {
                1 + \prn{L-1}
            }
        \right)
    }\right]
    =
    \frac{L}{L}
    \sum_{i=1}^{\subnetNum^*}
        \abs{\ouralpha{*}{i}}^{\order}
    =
    \norm{\valpha^*}_{\order}^{\order}~.
    \end{split}
    \end{align}
    %
    
    As a conclusion, 
    \begin{align}
        \Eucost_L\prn{\params^*}
        \le
        \Eucost_L\prn{\params}
        =
        \norm{\valpha^*}_{\order}^{\order}~.
    \end{align}
  }
\end{description}

Overall we get the required equality:
    \begin{align}
        \Eucost_L\prn{\params^*}
        =
        \norm{\valpha^*}_{\order}^{\order}~.
    \end{align}
\end{proof}

\removed{
\paragraph{Proof for \thmref{thm:equivalenceOnFiniteSets}}
$ $ \newline
We first state the following lemma,
\begin{lemma}
[Equivalence of loss minimization on finite sample sets]
\label{lem:equivalenceOnFiniteSets}
The following 
two norm minimization problems are equivalent 
and hold
\begin{align*}
    \begin{aligned}
    {\inf_{\params}} & 
    \norm{\params}_{2}^{2}
    \\ 
    \st~& \func{\params}\prn{x_{\smplIdx}} = f\prn{x_{\smplIdx}},~\forall{\smplIdx}
    \\
    \\
    \end{aligned}
    \begin{aligned}
    = L\cdot
    {\inf_{
        k, 
        \left\{\unit_i\right\}_{i=1}^{\subnetNum},
        \valpha
    }} & \norm{\valpha}_{\order}^{\order}
    \\ 
    \st~& \func{
        \valpha, \left\{\unit_i\right\}_{i=1}^{\subnetNum}
    }\!\prn{x_{\smplIdx}} = f\prn{x_{\smplIdx}},~\forall{\smplIdx}
    \\
    & \unit_i\in\sphere,~\forall{i\in\iter{\subnetNum}}
    \end{aligned}
\end{align*}
\end{lemma}

\begin{proof}
This is the same statement as in \thmref{prop:complexityRatio}
but on a finite sample set rather than on the entire domain of $f$.
The proof of \thmref{prop:complexityRatio} works the same,
since we show there that our constructed solutions
always preserve feasibility
(whether we demand the networks are equal to the functions over 
the entire domain of $f$,
or whether we only demand it over a finite set of $N$ points).
\end{proof}

Now we are ready for the main proof,
\begin{proof}
[Proof for \thmref{thm:equivalenceOnFiniteSets}]
We start by using the fact that optimizing over the network parameters
is equivalent to optimizing over the network possible outputs,
when constraining the network parameters to yield such outputs.
This allows us to consider the following problem equivalent 
to \eqref{eq:normalizedLossMinimization},
\begin{align}
\begin{aligned}
{{\inf_{\params}}~}
& \prn{
    \sum_{\smplIdx=1}^{\samples} 
    \loss\prn{\func{\params}\prn{x_{\smplIdx}}, y_{\smplIdx}}
    +
    \frac{\lambda}{L}
    \norm{\params}_2^2
    }
\end{aligned}
=
\min_{o_1,\dots,o_N}
\prn
{
    \sum_{\smplIdx=1}^{\samples} 
    \loss\prn{o_{\smplIdx}, y_{\smplIdx}}
    +
    \frac{\lambda}{L}
    \begin{aligned}
        \\
        \temp{\inf_{\params}} & \norm{\params}_2^2
        \\ 
        \st & \func{\params}\prn{x_{\smplIdx}} = o_{\smplIdx},~\forall{n\in\iter{\samples}
        }
        \hspace{-.1in}
    \end{aligned}
}%
~,
\end{align}
and the following problem equivalent to \eqref{eq:finiteConvexLossMinimization},
\begin{align}
\begin{aligned}
{{\inf_{
    k, 
    \left\{\unit_i\right\}_{i=1}^{\subnetNum},
    \valpha
}}~}
& \prn{
    \sum_{\smplIdx=1}^{\samples} 
    \loss\prn{
        \func{
            k, \left\{\unit_i\right\}_{i=1}^{\subnetNum},
            \valpha
        }{x_{\smplIdx}}, y_{\smplIdx}}
    +
    \lambda 
    \norm{\valpha}_{\order}^{\order}
}
\\
\st &
\unit_i \in \sphere,~\forall{i\in\iter{\subnetNum}}
\\
=
\min_{o_1,\dots,o_N}
&
\prn
{
    \sum_{\smplIdx=1}^{\samples} 
    \loss\prn{o_{\smplIdx}, y_{\smplIdx}}
    +
    \lambda
    \begin{aligned}
        \\
        \\
        \temp{\inf_{
            \valpha, \left\{\unit_i\right\}_{i=1}^{\subnetNum}
        }} & \norm{\valpha}_{\order}^{\order}
        \\ 
        \st & \func{
            \valpha, \left\{\unit_i\right\}_{i=1}^{\subnetNum}
        }\prn{x_{\smplIdx}} = o_{\smplIdx},~\forall{n\in\iter{\samples}
        }
        \hspace{-.15in}
        \\
        & \unit_i \in \sphere,~\forall{i\in\iter{\subnetNum}}
        \hspace{-.15in}
    \end{aligned}
}
~.
\end{aligned}
\end{align}

Using the above  \lemref{lem:equivalenceOnFiniteSets}
where we set $f\prn{x_{\smplIdx}}=o_{\smplIdx}$,
we see that the optimal value of 
the inner optimization problems
are equal up to a factor of $L$,
and the theorem follows immediately.
\end{proof}
}

\removed{
\subsection{Proof for \lemref{lem:positivityConstraint}}
\inote{If the lemma is removed,
remove this as well.}
\begin{proof}[sketch]
\label{app:positivityConstraint}
\todo{finish}
Given an optimal solution $\valpha^*$ to \eqref{eq:bridgeMeasure},
construct a solution $\tilde{\valpha}$ for \eqref{eq:withPosConstraints},
where $\tilde{\valpha}\prn{\weights, s}=\abs{\valpha\prn{\weights}}$
if $s=\sign\prn{\valpha\prn{\weights}}$ and $0$ otherwise. 
Clearly, 
$\norm{\valpha}_{\order}
=\norm{\tilde{\valpha}}_{\order}$,
and $\tilde{h}_{\tilde{\valpha}, \left\{\unit_i\right\}_{i=1}^{\subnetNum}}^{\subnetNum} 
=
{h}_{{\valpha}, \left\{\unit_i\right\}_{i=1}^{\subnetNum}}^{\subnetNum}$.

Given an optimal solution $\tilde{\valpha}^*$ to \eqref{eq:withPosConstraints}
construct a solution $\hat{\valpha}$ for \eqref{eq:bridgeMeasure},
where $\hat{\valpha}_{\weights} 
= 
\tilde{\valpha}^{*}_{\weights,1} - \tilde{\valpha}^{*}_{\weights, -1}$.
Again, it is clear that 
$\tilde{h}_{\tilde{\valpha}, \left\{\unit_i\right\}_{i=1}^{\subnetNum}}^{\subnetNum} 
=
{h}_{{\hat\valpha}, \left\{\unit_i\right\}_{i=1}^{\subnetNum}}^{\subnetNum}$.
Moreover, $\forall{\weights\in\sphere}$:
\begin{align}
\abs{\hat{\valpha}_{\weights}}^{\order}
=
\abs{\tilde{\valpha}^{*}_{\weights,1} - \tilde{\valpha}^{*}_{\weights, -1}}^{\order}
\le
\prn{
    \abs{\tilde{\valpha}^{*}_{\weights,1}} + \abs{\tilde{\valpha}^{*}_{\weights, -1}}
}^{\order}
\le
\abs{\tilde{\valpha}^{*}_{\weights,1}}^{\order} + \abs{\tilde{\valpha}^{*}_{\weights, -1}}^{\order}~,
\end{align}
\inote{cite the lecture notes, or (better) find a paper proving the last inequality.}
which means $\norm{\hat{\valpha}}_{\order}^{\order}
\le
\norm{\tilde{\valpha}^{*}}_{\order}^{\order}$.
Another corollary: 
The last inequality must hold with equality \inote{explain},
and so one the two values must be $0$ for a given $\weights$ \inote{explain also}.
\end{proof}
}

\paragraph{Proof for \thmref{thm:lastLayerSupportBound}}
$ $ \newline
\begin{proof}
\removed{
To prove this Theorem,
it is sufficient to show that if the optimal solution
to \eqref{eq:finiteConvexLossMinimization} is attainable, then there is a finite
with a support size at most $N$.}

When $L=2 \Rightarrow \frac{2}{L} = 1$,
our minimization problem is a standard $\ell_1$-penalized
problem.
\citet[Theorem~2]{rosset2007l1} 
use Carathéodory's theorem
and prove that there exists
an optimal solution with a finite support size,
which is at most $\samples+1$ (using Carathéodory's theorem).
\citet[Lemma~14]{tibshirani2013lasso} 
demonstrate how an iterative procedure 
(using a technique similar to what we do below)
can zero out an (attainable) optimal solution with a finite support,
until an optimal solution with a support size at most $N$
is attained.

We therefore narrow our proof to deeper networks
where $L\ge3 \Rightarrow \frac{2}{L} < 1$
and the norm in the objective function is no longer convex
(but quasi-convex).
Following a proof by 
\citet{petrov2018support},
we are able to show that when $L\ge 3$,
not only there exists an optimal solution with a finite support
of size at most $N$,
but \textbf{all} optimal solution are such.

We start by assuming the contrary -- 
there exists an optimal solution $\params^*$,
with a set of unit vectors 
$\left\{\unit_1,\dots,\unit_{\samples+1}\right\}\subseteq\params^*$
such that  w.l.o.g.
$\norm{\valpha}_0 = \samples + 1$ (that is, $\ouralpha{*}{i}
\neq 0,
\forall{i\in\iter{\samples+1}}$).
\removed{For ease of notation,
we denote $\ouralpha{*}{i}\triangleq \ouralpha{*}{}\left(\unit_{i}\right)$.}
We wish to find a non-zero vector ,
$\vb \in \R^{\samples+1}$
such that 
$\forall{i\in\iter{\samples+1},n\in\iter{\samples}}:
\vb_{i}
\subfunc{\unit_i}{x_{\smplIdx}}=0$.
Notice we have only $\samples$ constraints 
and $\samples+1$ variables, 
meaning we get an homogeneous underdetermined system.
We can thus always choose such a vector 
$\vbeta \in \R^{\subnetNum^*}$
such that
\begin{align}
\vbeta_i
=
 \begin{cases}
   \vb_{i} & 
    {i\in\iter{\samples+1}}\\
   0 &\text{otherwise} \\ 
 \end{cases}
\end{align}
and
$\sum_{i=1}^{\samples+1}
    \vbeta_i
    \subfunc{\unit_i}{x_{\smplIdx}}=0, 
    \forall{n\in\iter{\samples}}$,
which means the solutions
$\ouralpha{*}{} + \rho\vbeta, \forall{\rho\in\mathbb{R}}$
have the same network output
as $\ouralpha{*}{}$.
Now choose $\rho>0$ small enough, such that 
$\forall{i\in\iter{\samples+1}}$:
\begin{align}
\label{eq:rhoCondition}
    \abs{\ouralpha{*}{i}}
    - 
    \rho\abs{\vbeta_{i}}
    =
    s_{i}\ouralpha{*}{i}
    - 
    \rho\abs{\vbeta_{i}}
    > 0~,
\end{align}
where $s_{i} \triangleq \sign{\ouralpha{*}{i}}$.
The function ${z}^{\frac{2}{L}}$ is concave $\forall{z\ge0}$,
and so we apply the Jensen inequality to get
\begin{align}
\begin{split}
    {\norm{\ouralpha{*}{}}}_{\order}^{\order}
    \ge &
    \sum_{i}^{\samples+1}
    \abs{\ouralpha{*}{i}} ^ {\frac{2}{L}}
     = 
    \sum_{i}^{\samples+1}
    \left(s_{i}\ouralpha{*}{i}\right) ^ {\frac{2}{L}}
    \\
     = &
    \sum_{i}^{\samples+1}
    \left(
        \frac{1}{2}
        \left(
        s_{i}\ouralpha{*}{i} + \rho\vbeta_{i}
        \right)
        +
        \frac{1}{2}
        \left(
        s_{i}\ouralpha{*}{i} - \rho\vbeta_{i}\right)
    \right)
    ^{\frac{2}{L}}
    \\
    >  & 
    \frac{1}{2}
    \sum_{i}^{\samples+1}
    \left(
    {\underbrace{
    \left(s_{i}\ouralpha{*}{i} + \rho\vbeta_{i}\right)
    }_{>0}}^ {\frac{2}{L}} 
    +
    {\underbrace{
    \left(s_{i}\ouralpha{*}{i} - \rho\vbeta_{i}\right)}_{
    > 0
    }} ^ {\frac{2}{L}}
    \right)
    \\
    =  & 
    \frac{1}{2}
    \sum_{i}^{\samples+1}
    \left(
    {
    \abs{\ouralpha{*}{i} + s_{i}\rho\vbeta_{i}}
    }^ {\frac{2}{L}} +
    {
    \abs{\ouralpha{*}{i} - s_{i}\rho\vbeta_{i}}} ^ {\frac{2}{L}}
    \right)
    \\
    =  & 
    \frac{1}{2}
    \sum_{i}^{\samples+1}
    \left(
    {
    \abs{\ouralpha{*}{i} + \rho\vbeta_{i}}
    }^ {\frac{2}{L}} +
    {
    \abs{\ouralpha{*}{i} - \rho\vbeta_{i}}} ^ {\frac{2}{L}}
    \right)
    \\
    = & 
    \frac{1}{2}
    \left(
    \norm{
        \ouralpha{*}{} + \rho\vbeta}_{\order}^{\order}
    +
    \norm{
        \ouralpha{*}{} - \rho\vbeta}_{\order}^{\order}
    \right)~.
\end{split}
\end{align}
Notice we used the strict Jensen inequality,
since neither the function $z^{\frac{2}{L}}$ is linear,
nor are the two terms equal 
(following our choice of $\rho$ in \eqref{eq:rhoCondition}).

Finally, we notice the above imply that one of the two solutions
$\ouralpha{*}{} \pm \rho\vbeta$ must have a strictly smaller norm than $\ouralpha{*}{}$.
As a result of our choice of $\vbeta$, 
all three solutions have the same loss,
\ie
\begin{align}
    \sum_{\smplIdx=1}^{\samples}
    \loss\prn{\func{\ouralpha{*}{}}(x_{\smplIdx}}, y_{\smplIdx})
    =
    \sum_{\smplIdx=1}^{\samples}
    \loss\prn{\func{\ouralpha{*}{}+\rho\vbeta}(x_{\smplIdx}), y_{\smplIdx}}
    =
    \sum_{\smplIdx=1}^{\samples}
    \loss\prn{\func{\ouralpha{*}{}-\rho\vbeta}(x_{\smplIdx}}, y_{\smplIdx})~.
\end{align}
The above necessarily mean that at least one of these
two solutions we constructed
has an objective value strictly smaller 
than the objective value of $\ouralpha{*}{}$,
in contradiction to its optimality.
\end{proof}

\section{Proof of Claim \ref{claim:vanishing Hessian} \label{sec:vanishing Hessian proof}}

\begin{proof}
If $\left\Vert \mathbf{x}\right\Vert >b$ then,
\begin{align*}
 & \int_{\mathbb{S}^{d-1}}d\wvec\left[\mathbf{w}^{\top}\mathbf{x}+b\right]_{+}\\
 & \overset{\left(1\right)}{=}\int_{\mathbb{S}^{d-1}}d\wvec\left[\left\Vert \mathbf{x}\right\Vert w_{1}+b\right]_{+}\\
 & \overset{\left(2\right)}{=}\left\Vert \mathbf{x}\right\Vert \int_{-1}^{1}dw_{1}\left[w_{1}+\frac{b}{\left\Vert \mathbf{x}\right\Vert }\right]_{+}\int_{\wvec^{\prime}\in\mathbb{S}^{d-1}: w^{\prime}_1=w_1 } \prod_{i=2}^{d}dw_{i}^{\prime}\\
 & \overset{\left(3\right)}{=}\left\Vert \mathbf{x}\right\Vert \int_{\max\left[-1,-\frac{b}{\left\Vert \mathbf{x}\right\Vert }\right]}^{1}dw_{1}\left(w_{1}+\frac{b}{\left\Vert \mathbf{x}\right\Vert }\right)S_{d-1}\left(1-w_{1}^{2}\right)^{\frac{d-1}{2}}\\
 & \overset{\left(4\right)}{=}\left\Vert \mathbf{x}\right\Vert S_{d-1}\int_{-\frac{b}{\left\Vert \mathbf{x}\right\Vert }}^{1}dw_{1}w_{1}\left(1-w_{1}^{2}\right)^{\frac{d-1}{2}}+S_{d-1}\int_{-\frac{b}{\left\Vert \mathbf{x}\right\Vert }}^{1}dw_{1}\left(1-w_{1}^{2}\right)^{\frac{d-1}{2}}\\
 & \overset{\left(5\right)}{=}\left\Vert \mathbf{x}\right\Vert S_{d-1}\int_{-\frac{b}{\left\Vert \mathbf{x}\right\Vert }}^{1}dw_{1}w_{1}\left(1-w_{1}^{2}\right)^{\frac{d-1}{2}}+C+f_{\mathrm{odd}}\left(\frac{b}{\left\Vert \mathbf{x}\right\Vert }\right)\\
 & =\left\Vert \mathbf{x}\right\Vert \frac{S_{d-1}}{d+1}\left(1-\frac{b^{2}}{\left\Vert \mathbf{x}\right\Vert ^{2}}\right)^{\frac{d+1}{2}}+C+f_{\mathrm{odd}}\left(\frac{b}{\left\Vert \mathbf{x}\right\Vert }\right)
\end{align*}
In $\left(1\right)$ we assume $d>1$ and, WLOG, that $\mathbf{x}$
is in direction $\left(1,0,\dots,0\right)$, in $\left(2\right)$
we integrate over the surface area of a $d-2$ dimensional sphere, in $\left(3\right)$ we denoted $S_{d}$ as the
surface area of the $d$-sphere, in $\left(4\right)$ we assume that $\left\Vert \mathbf{x}\right\Vert >b$,
and in $\left(5\right)$ we used the fact that the integral of an
even function is equal to the sum of a constant $C$ with and odd function
$f_{\mathrm{odd}}$. 
Therefore, for $\left\Vert \mathbf{x}\right\Vert >b$
\begin{align*}
h(\xvec) & =\int_{\mathbb{S}^{d-1}}d\wvec\left(\left[\mathbf{w}^{\top}\mathbf{x}+1\right]_{+}-2\left[\mathbf{w}^{\top}\mathbf{x}\right]_{+}+\left[\mathbf{w}^{\top}\mathbf{x}-1\right]_{+}\right)\\
 & =2\left\Vert \mathbf{x}\right\Vert \frac{S_{d-1}}{d+1}\left[\left(1-\frac{1}{\left\Vert \mathbf{x}\right\Vert ^{2}}\right)^{\frac{d+1}{2}}-1\right]\\
 & =S_{d-1}\frac{1}{\left\Vert \mathbf{x}\right\Vert }+O\left(\frac{1}{\left\Vert \mathbf{x}\right\Vert ^{3}}\right)\,.
\end{align*}
From here it is straightforward to see that 
\[
\nabla h(\xvec)\sim-\frac{\mathbf{x}}{\left\Vert \mathbf{x}\right\Vert ^{3}}
\]
\[
\nabla^{2}h(\xvec)\sim\frac{1}{\left\Vert \mathbf{x}\right\Vert ^{3}}\left(3\frac{\mathbf{x}\mathbf{x}^{\top}}{\left\Vert \mathbf{x}\right\Vert ^{2}}-\mathbf{I}\right)
\]
Therefore, for some constants $r_0,C$, and any norm,
\begin{align*}
\frac{1}{r^{d-1}}\int_{\left\Vert \xvec\right\Vert \leq r}d\mathbf{x}\left\Vert \nabla^{2}h(\xvec)\right\Vert  & \leq\frac{C}{r^{d-1}}\int_{r_0}^{r}u^{d-1}\frac{1}{u^{3}}du=\frac{C}{r^{d-1}}\int_{r_0}^{r}u^{d-4}du\\
 & =\frac{C}{r^{d-1}}r^{d-3}=\frac{C}{r^{2}}
\end{align*}
Which vanishes as $r\rightarrow\infty$.
\end{proof}

\end{document}